\pdfoutput=1
\documentclass[11pt]{article}

\PassOptionsToPackage{table}{xcolor}
\usepackage[preprint]{acl}

\usepackage{times}
\usepackage{latexsym}
\usepackage{amsmath}
\usepackage{amsfonts}
\usepackage{amssymb}
\usepackage{booktabs}
\usepackage{multirow}
\usepackage[table]{xcolor}
\usepackage{amsthm}
\usepackage{subcaption}
\usepackage{wrapfig}
\usepackage{enumitem}
\newtheorem{proposition}{Proposition}

\newcommand{\boldhdr}[1]{
    \vspace{0.08cm}
    \noindent\textbf{#1}.
}
\newcommand{\rateinline}[2]{#1\color{gray}{\tiny$\pm$ #2}}

\usepackage[T1]{fontenc}

\usepackage[utf8]{inputenc}

\usepackage{microtype}

\usepackage{inconsolata}

\usepackage{graphicx}

\title{Diversity-Aware Reverse Kullback-Leibler Divergence for Large Language Model Distillation}

\author{
    Hoang-Chau Luong$^{1}$ \quad 
    Dat Ba Tran$^{2}$ \quad
    Lingwei Chen$^{1}$ \\
    $^{1}$Rochester Institute of Technology, Rochester, NY, USA \\
    $^{2}$Rowan University, Glassboro, NJ, USA \\
    \texttt{cl6300@rit.edu, trandb@rowan.edu, lwcics@rit.edu} \\
}

\begin{document}
\maketitle

\begin{abstract}
Reverse Kullback–Leibler (RKL) divergence has recently emerged as the preferred objective for large language model (LLM) distillation, consistently outperforming forward KL (FKL), particularly under large vocabularies and significant teacher–student capacity mismatch, where RKL focuses learning on dominant modes rather than enforcing dense alignment. However, RKL introduces a structural limitation that drives the student toward overconfident predictions. We first analyze RKL by decomposing its gradients into target and non-target components, and show that non-target gradients consistently push the target logit upward even when the student already matches the teacher, thereby reducing output diversity. Also, RKL provides weak supervision over non-target classes, leading to poor tail alignment. We thus propose diversity-aware RKL (DRKL), which removes this gradient effect and strengthens non-target supervision while retaining the optimization benefits of RKL. Experiments across datasets and model families show that DRKL consistently outperforms FKL, RKL, and other strong distillation objectives, achieving a superior fidelity–diversity trade-off. Our code is available at \url{https://github.com/lhchau/drkl-diversity-aware-kl}.
\end{abstract}

\section{Introduction}

Knowledge distillation (KD)~\citep{hinton2015distilling} aims to transfer knowledge from large language models (LLMs) to smaller, more efficient students~\citep{wu2024weight}, typically by minimizing a Kullback–Leibler (KL) divergence between their output distributions. While forward KL (FKL) has been the standard objective in early distillation work~\citep{hinton2015distilling, cho2019efficacy, mirzadeh2020improved, jin2023multi, sun2024logit}, recent LLM distillation methods increasingly favor reverse KL (RKL) due to its superior empirical performance~\citep{gu2023minillm, gu2025miniplm, agarwal2024policy}. This advantage is commonly attributed to the \textit{mode-seeking} behavior of RKL, which concentrates probability mass on dominant teacher modes, whereas FKL is more \textit{mass-covering} and encourages matching the full distribution~\citep{chan2022greedification}. However, prior work~\citep{wu2025rethinking} suggests that FKL and RKL converge to the same optimum under sufficient optimization, implying that their differences mainly manifest during early training.

In this work, we argue that the advantage of RKL arises from its alignment with the \textit{structure} of LLMs, rather than being limited to the early stages of training. The theoretical equivalence between FKL and RKL relies on assumptions (e.g., sufficient student capacity, exact optimization to the global optimum) that rarely hold in modern LLM distillation, particularly under large vocabularies and significant teacher–student capacity mismatch. In such settings, FKL enforces probability matching across the entire vocabulary, requiring the student to approximate the teacher's long-tail distribution, which is inherently challenging for capacity-limited models. In contrast, RKL focuses learning on high-probability tokens while suppressing low-probability mass, effectively reducing the complexity of the alignment problem. This leads to the better performance of RKL observed in practice~\citep{gu2023minillm,kim2024promptkd}.

Despite these advantages, RKL introduces fundamental limitations. First, it creates a shortcut that rapidly minimizes non-target loss by increasing the confidence of the target class, without properly aligning non-target distributions of the teacher and student. Knowledge transfer is thus incomplete, since RKL can remain low despite poor non-target alignment. Second, by concentrating probability mass on dominant tokens, it systematically reduces the entropy of the student distribution, resulting in overconfident predictions and diminished output diversity. Empirically, students trained with RKL often exhibit higher confidence than their teachers, creating a clear quality–diversity imbalance.

To better understand these behaviors, we theoretically and empirically analyze RKL by decomposing the gradients into target and non-target components. While similar decompositions have been explored for FKL in vision settings~\citep{zhao2022decoupled, wei2024scaled, cui2024decoupled}, the optimization dynamics of RKL in large-vocabulary LLM distillation remain largely unexplored.  Our analysis reveals a distinct mechanism: non-target gradients consistently push the target logit upward, continuing to increase it even after the student's target probability matches that of the teacher, which provides a mechanism consistent with both limitations. Motivated by this observation, we propose \textbf{Diversity-aware RKL (DRKL)}, which removes the contribution of non-target gradients to the target logit and strengthens non-target supervision while preserving the optimization benefits of RKL. Extensive experiments across multiple datasets and model families show that DRKL consistently improves the fidelity–diversity trade-off.

Our contributions are summarized as follows:
\begin{itemize}[leftmargin=12pt,noitemsep,topsep=0pt]
    \item We argue that the empirical advantage of RKL over FKL in LLM distillation arises from its alignment with the structural properties of LLMs. Specifically, under large vocabularies and teacher–student capacity mismatch, RKL simplifies alignment by concentrating learning on dominant modes (in Section~\ref{sec:equivalence_breakdown}).

    \item We theoretically identify key limitations of RKL (in Section~\ref{sec:reverse_analysis}): RKL creates a shortcut that minimizes non-target loss without improving alignment over non-target classes (in Proposition~\ref{prop:rev_dkd_statement}); non-target gradients with respect to the target logit are consistently negative, which pushes the target logit upward and leads to overconfident and low-diversity predictions (in Proposition~\ref{prop:rev_tckd_instability}).

    \item We propose DRKL, which removes the non-target gradients with respect to the target logit and strengthens non-target supervision (in Section~\ref{sec:drkd}). Experiments across datasets and model families demonstrate that DRKL consistently improves performance and fidelity–diversity trade-off (in Section~\ref{sec:experiments}).
\end{itemize}

\section{Related Work}
\vspace{-0.1cm}

\boldhdr{KD for LLM distillation}
KD~\citep{hinton2015distilling} is widely used to compress LLMs into compact students~\citep{wu2024weight}. Existing approaches can be broadly categorized into black-box and white-box distillation. Black-box KD relies solely on the outputs of the teacher model. For example, sequence-level KD~\citep{kim2016sequence} trains the student on responses generated by the teacher. In contrast, white-box KD leverages richer teacher signals, typically by aligning token-level output distributions between teacher and student. This alignment is commonly achieved by minimizing a KL divergence between their logits during token generation, using either FKL~\citep{hinton2015distilling,kim2023token} or RKL~\citep{gu2023minillm, agarwal2024policy, kim2024promptkd}.

\boldhdr{KL objectives in KD}
FKL was introduced for KD by~\citet{hinton2015distilling} and has since been widely adopted across vision and language tasks. Recent LLM distillation methods increasingly employ RKL due to its stronger empirical performance~\citep{gu2023minillm, agarwal2024policy}. Several works study the relationship between FKL and RKL. \citet{wu2025rethinking} show that the two objectives converge to the same optimum under sufficient optimization and propose AKL to balance their early-training behavior. Similarly, AB~\citep{wang2025abkd} balances mode-seeking and mean-seeking effects, while SFKL and SRKL~\citep{ko2024distillm} alleviate teacher–student mismatch by interpolating their distributions. Another line of work decomposes FKL into target and non-target components~\citep{zhao2022decoupled, wei2024scaled, cui2024decoupled}, showing that these terms contribute differently to training dynamics. These methods primarily address vanishing gradients for non-target classes in vision tasks. Unlike FKL, the non-target weight in RKL depends on the student, opening an additional gradient pathway to the target logit. We leverage this decomposition to uncover a structural limitation of RKL in large-vocabulary LLM distillation: non-target gradients exert persistent upward pressure on the target logit, driving the student toward overconfident and low-diversity predictions.

\begin{figure*}[t]
\vspace{-0.2cm}
    \centering
    \begin{subfigure}[t]{0.48\linewidth}
        \centering
        \includegraphics[width=\linewidth]{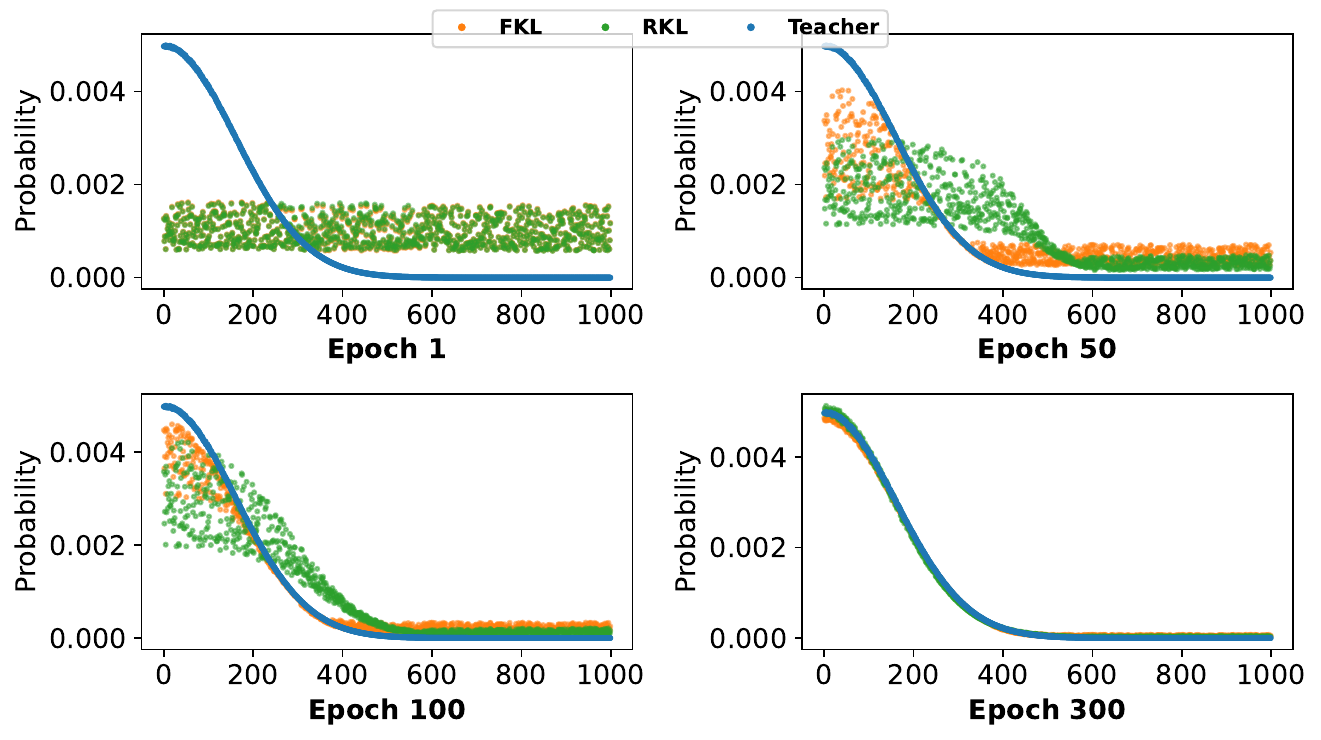}
        \caption{A small output space $V=1{,}000$.}
        \label{fig:small_vocab}
    \end{subfigure}
    \hfill
    \begin{subfigure}[t]{0.48\linewidth}
        \centering
        \includegraphics[width=\linewidth]{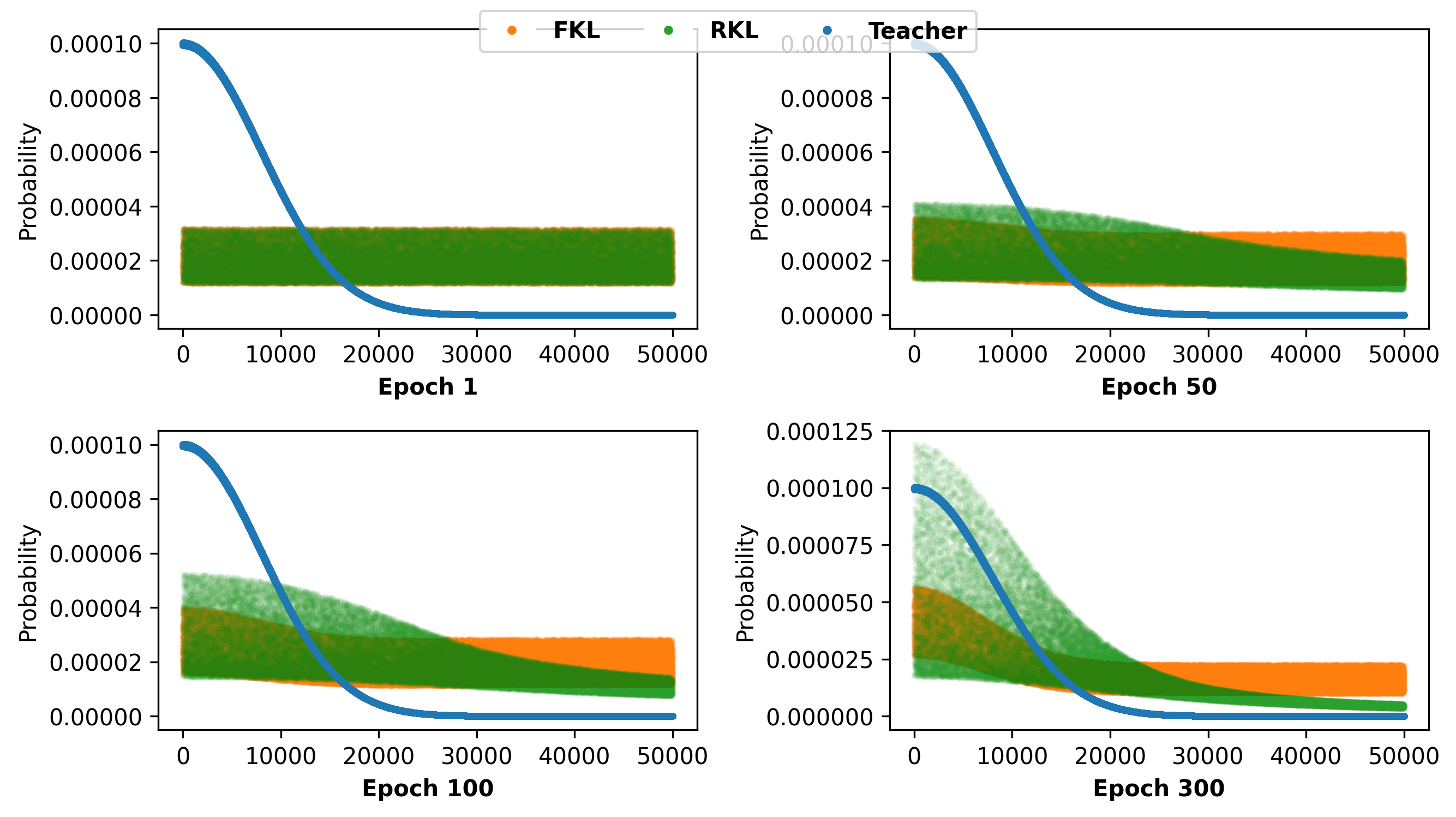}
        \caption{A large output space $V=50{,}000$.}
        \label{fig:large_vocab}
    \end{subfigure}
    \vspace{-0.2cm}
    \caption{FKL and RKL when fitting a teacher distribution under different output sizes. The x-axis shows the class index, and the y-axis shows the corresponding output probability.
    As the number of classes increases, the optimization difficulty grows substantially. The teacher is a fixed, discretized, monotone, heavy-tailed Gaussian over $V$ classes, and the student is an unconstrained $V$-dimensional logit vector followed by softmax.}
    \label{fig:longtail_kl_dynamics}
    \vspace{-0.4cm}
\end{figure*}

\section{Preliminaries}
\label{sec:analysis_main}
\vspace{-0.1cm}

\textbf{Problem formulation}.
Let $\mathbf{x} = (x_1, \dots, x_L)$ denote an input sequence of length $L$, and $\mathbf{y} = (y_1, \dots, y_T)$ denote an output sequence of length $T$. At decoding step $t$, given the prefix $\mathbf{y}_{<t} := (y_1, \dots, y_{t-1})$, a teacher model produces a predictive distribution $p^t(\cdot) := p(\cdot \mid \mathbf{y}_{<t},  \mathbf{x}) \in (0, 1)^{V}$, where $V$ is the vocabulary size. A student model parameterized by $\theta$ produces
$q^{t}(\cdot) := q_\theta(\cdot \mid \mathbf{y}_{<t}, \mathbf{x}) \in (0, 1)^{V}$.
KD aligns student with teacher by minimizing a divergence between their predictive distributions, aggregated over decoding steps:
\begin{equation}
    \mathcal{L}_{\mathrm{KD}}(\theta)
    = \sum_{t=1}^T D (p^t \| q^{t} ).
\end{equation}
In this work, we analyze the KL divergence at a fixed decoding step $t$. We thus omit the time index $t$ for simplicity. Two common choices for divergence measure $D$ are FKL and RKL, which induce different training dynamics and empirical performance~\citep{kim2016sequence, agarwal2024policy}.

\boldhdr{Forward and reverse KL}
The standard KD objective minimizes FKL divergence \citep{hinton2015distilling, kim2016sequence}, which is defined as:
\begin{equation}
    D_{\mathrm{KL}}(p\|q)=\sum_{j=1}^V p_j\log\frac{p_j}{q_j}.
\label{eq:fkl_def}
\end{equation}
Recent LLM distillation methods~\citep{gu2023minillm, agarwal2024policy} instead employ RKL divergence, often reporting superior empirical performance compared to FKL. The RKL is defined as:
\begin{equation}
    D_{\mathrm{KL}}(q\|p)=\sum_{j=1}^V q_j\log\frac{q_j}{p_j}.
\label{eq:rkl_def}
\end{equation}
FKL penalizes under-coverage of the teacher distribution, encouraging probability mass to be spread across many tokens, including the heavy tail. In contrast, RKL penalizes assigning probability to low-likelihood teacher tokens, concentrating mass on dominant modes and producing sharper predictions. While this intuition captures their qualitative differences, it does not fully explain the pronounced performance gap observed in modern LLM distillation. In particular, in regimes with extremely large vocabularies and significant teacher–student capacity mismatch, the two objectives exhibit fundamentally different optimization dynamics. We analyze this gap in Section~\ref{sec:method}.

\begin{figure*}[t]
    \centering
    \vspace{-2mm}
    \begin{subfigure}[t]{0.32\linewidth}
        \centering
        \includegraphics[width=\linewidth]{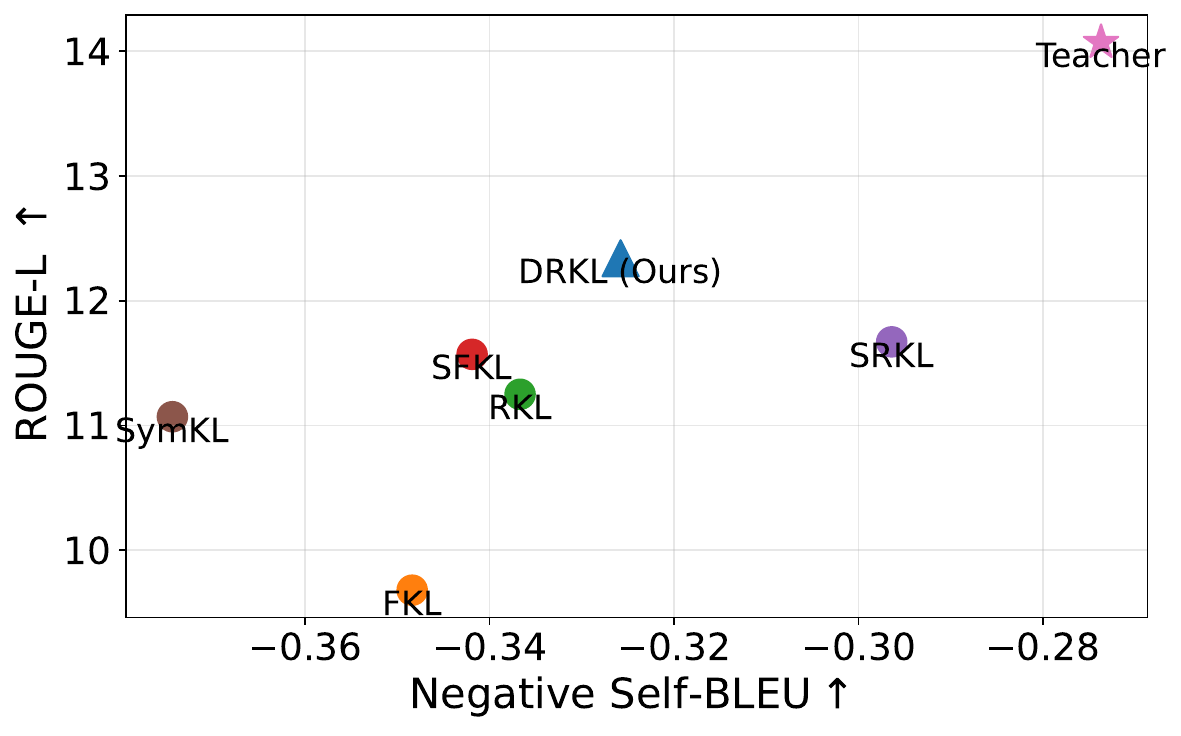}
        \vspace{-4mm}
        \caption{ROUGE-L vs. Neg Self-BLEU}
        \label{fig:fide_neg_self_bleu}
    \end{subfigure}
    \hfill
    \begin{subfigure}[t]{0.32\linewidth}
        \centering
        \includegraphics[width=\linewidth]{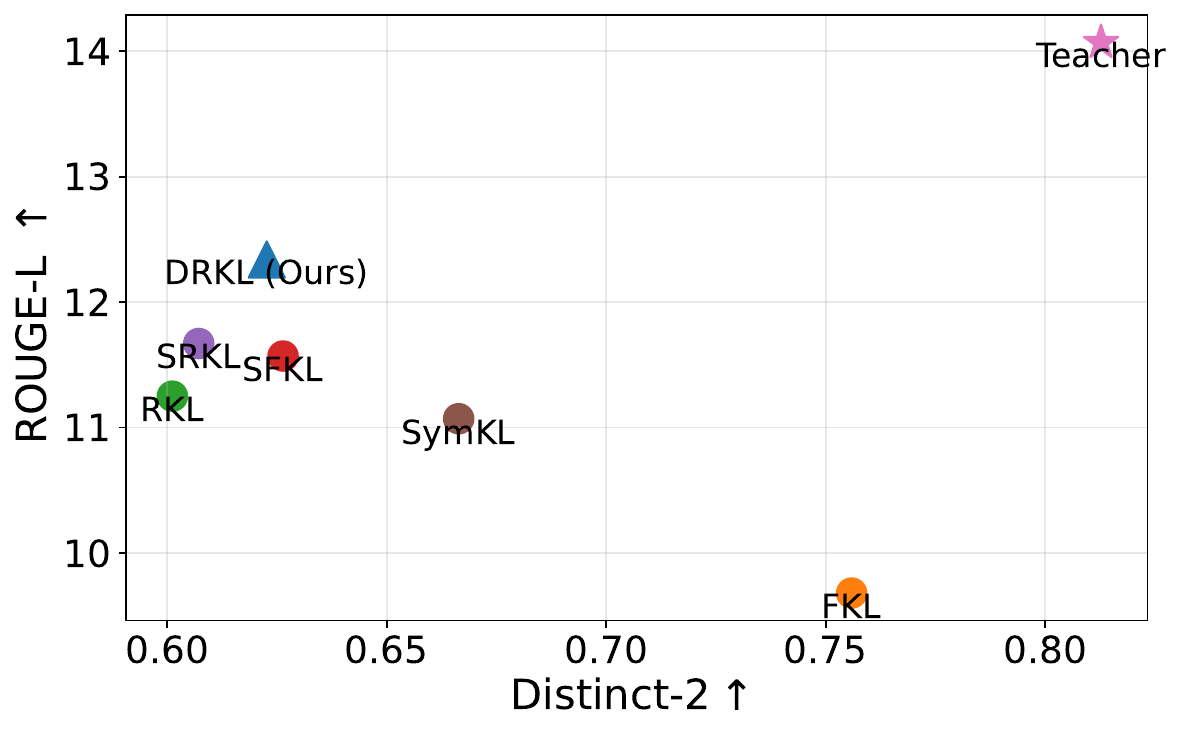}
        \vspace{-4mm}
        \caption{ROUGE-L vs. Distinct-2}
        \label{fig:distinct2}
    \end{subfigure}
    \hfill
    \begin{subfigure}[t]{0.32\linewidth}
        \centering
        \includegraphics[width=\linewidth]{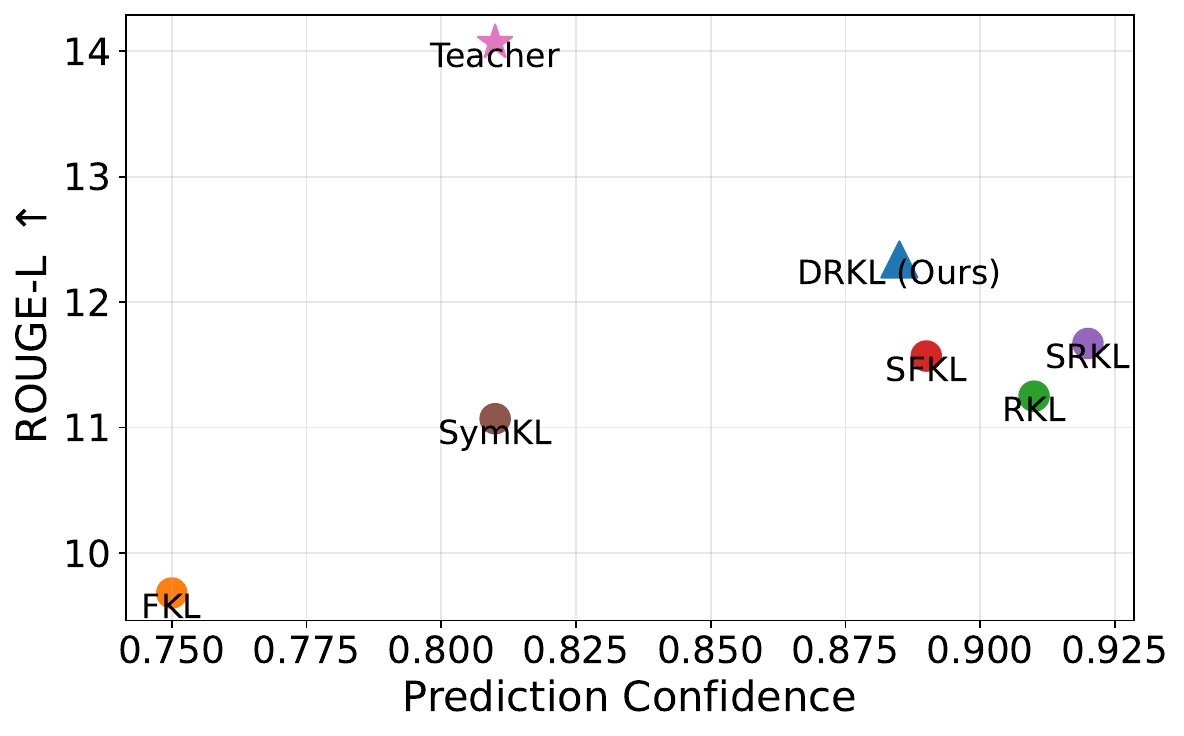}
        \vspace{-4mm}
        \caption{ROUGE-L vs. Pred conf}
        \label{fig:rouge_conf}
    \end{subfigure}
    \vspace{-0.2cm}
    \caption{
        (a, b) Fidelity vs.\ diversity: RKL reduces diversity (Negative Self-BLEU and Distinct-2), while DRKL achieves a better balance across methods.
        (c) ROUGE-L vs.\ prediction confidence: RKL produces overconfident predictions without improving ROUGE-L, while DRKL reduces overconfidence.
    }
    \label{fig:main_tradeoff}
    \vspace{-0.4cm}
\end{figure*}

\vspace{-0.1cm}
\section{Analyses and Methodology}
\label{sec:method}
\vspace{-0.1cm}

\subsection{Structural Advantages of RKL}
\label{sec:equivalence_breakdown}
\vspace{-0.1cm}

FKL and RKL are expected to converge to the same solution under sufficient optimization~\citep{wu2025rethinking}. However, this result assumes sufficient student model capacity and exact optimization to the global optimum, such that the student can perfectly match the teacher distribution. These conditions often do not hold in modern LLM distillation. While~\citet{wu2025rethinking} attribute the practical gap mainly to limited training steps, we argue that it is also caused by two structural factors: large output spaces and significant teacher--student capacity mismatch. Under these conditions, RKL converges faster by focusing on dominant tokens, which simplifies alignment between teacher and student. In contrast, FKL enforces matching across the entire vocabulary, making optimization more difficult.

\boldhdr{Effect of large output spaces}
Modern LLMs operate over extremely large vocabularies, and their predictive distributions are heavy-tailed. In such settings, the choice of divergence used for distillation significantly affects the optimization dynamics. Let $z_{s,j} \in \mathbb{R}$ denote the student logit for the $j$-th class. The gradient of the FKL is computed as:
\begin{align}
    \nabla_{z_{s,j}} D_{\mathrm{KL}}(p \| q)
    = q_j - p_j .
\label{eq:fkl_grad}
\end{align}
This gradient provides supervision over all classes whenever the student and teacher distributions are misaligned. As a result, FKL produces dense signals across the vocabulary, with the number of active coordinates scaling with the vocabulary size $V$. In large output spaces, this requires matching many low-probability tokens, which spreads the optimization effort across the vocabulary. In contrast, the gradient of RKL is computed as:
\begin{equation}
    \nabla_{z_{s,j}} D_{\mathrm{KL}}(q\|p)
    =
    q_j
    \Big(
    \log \frac{q_j}{p_j}
    -
    D_{\mathrm{KL}}(q\|p)
    \Big),
    \label{eq:rkl_grad}
\end{equation}
which includes a scaling factor $q_j$ that suppresses gradient contributions from classes to which the student assigns negligible probability. Let $C \in (0,1)$ and define the active set
$
    \mathcal{A}_C = \{ j : q_j > C \}.
$
Only tokens in $\mathcal{A}_C$ receive substantial updates, while classes with $q_j \le C$ are effectively suppressed. As a result, the number of actively updated coordinates is much smaller than the full output space, $|\mathcal{A}_C| \ll V$. This behavior is illustrated in Figure~\ref{fig:longtail_kl_dynamics}: RKL quickly suppresses probability mass on tail tokens and focuses learning on high-probability classes. We interpret this as a key mechanism of RKL. By restricting optimization to a smaller set of tokens, it avoids matching the teacher over the full long-tail distribution, leading to more efficient optimization and better performance.

Figure~\ref{fig:longtail_kl_dynamics} shows optimization dynamics of FKL and RKL in a toy experiment with different vocabulary sizes. With a small vocabulary (1{,}000), FKL and RKL differ mainly
in early training but converge to similar solutions by the end of training (epoch 300), as shown in Figure~\ref{fig:small_vocab}. However, when the vocabulary size increases to 50{,}000, the difference becomes clear: RKL converges much faster than FKL (e.g., at epoch 300 in Figure~\ref{fig:large_vocab}). Figure~\ref{fig:longtail_kl_dynamics} thus shows the difference in training speed, while Appendix~\ref{appendix:toy_example} examines the converged solutions under explicit capacity mismatch.

\boldhdr{Implications under capacity mismatch}
These differences are further amplified by teacher--student capacity mismatch. In practice, the teacher often has much higher capacity, making it difficult for the student to match the full teacher distribution. FKL requires matching probabilities across the entire output space, including many low-probability tokens, making dense alignment difficult. In contrast, RKL focuses on the dominant modes of the distribution. Since gradients are weighted by the student probabilities, tokens with very small probability receive almost no updates. As a result, the number of active classes depends on the active set $\mathcal{A}_C$ of the student distribution rather than the vocabulary size:
$
|\mathcal{A}_C| \ll V .
$
This reduces the effective output space and makes the alignment problem easier when the student has limited capacity. We emphasize that this effective reduction serves as a mechanistic interpretation rather than a formal result, and our Propositions~\ref{prop:rev_dkd_statement}--\ref{prop:rev_tckd_instability} and DRKL do not rely on it.

\subsection{Gradient Dynamics of RKL}
\label{sec:reverse_analysis}
\vspace{-0.1cm}

\boldhdr{RKL produces overconfident students}
Despite its optimization advantages, RKL also leads to overconfident predictions and reduced diversity. Figures~\ref{fig:fide_neg_self_bleu} and~\ref{fig:distinct2} show this effect. We use ROUGE-L to measure fidelity, and Distinct-2~\citep{li2016diversity} and Negative Self-BLEU~\citep{zhu2018texygen} to measure diversity. Figure~\ref{fig:rouge_conf} plots ROUGE-L against prediction confidence.
As observed, RKL achieves higher ROUGE-L than FKL, indicating better alignment with the teacher. However, this comes with a clear drop in Distinct-2 (FKL 0.755 vs.\ RKL 0.600) and a large increase in prediction confidence (FKL 0.75 vs.\ RKL 0.91), which is even higher than the teacher. For Negative Self-BLEU, both methods show similar values ($-0.34$ for RKL and $-0.35$ for FKL). The teacher achieves high fidelity and diversity, indicating that this imbalance is caused by KL structure rather than a true trade-off.

\boldhdr{Why simple FKL--RKL interpolation fails}
A natural attempt to balance fidelity and diversity is to interpolate between FKL and RKL. However, this is often ill-conditioned due to the mismatch in gradient scales induced by the two KL directions. From Eq.~\eqref{eq:fkl_grad}, the gradient of FKL is bounded as $(q_j - p_j) \in (-1, 1)$, whereas the RKL gradient contains $\log(q_j / p_j)$, which can grow unbounded when $p_j \rightarrow 0$ with $q_j>0$. The combined gradient is thus dominated by RKL, making the interpolation coefficient ineffective in controlling the optimization.
A detailed analysis is provided in Appendix~\ref{appendix:grad_scale}.

\boldhdr{Target and non-target class alignment in RKL}
The overconfidence induced by RKL suggests a structural bias in its gradient updates. To characterize this behavior, we analyze RKL through a target and non-target class decomposition.
Prior work has applied similar decompositions to FKL~\citep{zhao2022decoupled, wei2024scaled, cui2024decoupled}, primarily in vision settings with moderate output dimensions. These studies focus on gradient vanishing effects for non-target classes. However, such analyses do not capture the structural regime of RKL in large-vocabulary LLM distillation. In this work, we show that under RKL, gradients from non-target classes consistently push the target logit upward. This interaction creates a confidence amplification mechanism: the target probability can continue increasing even after it matches or exceeds the teacher probability. As we show in the following analysis, this effect provides a mechanism consistent with the overconfidence observed in practice.

\begin{proposition}[Target and non-target decomposition of RKL]
\label{prop:rev_dkd_statement}
Let $p, q \in \mathbb{R}^{V}$ denote the teacher and student output probabilities among V classes, respectively, and let $m$ denote the index of the target class. Define $\tilde p_m = (p_m, 1 - p_m) \in \mathbb{R}^{2}$ and $\tilde q_m = (q_m, 1 - q_m) \in \mathbb{R}^{2}$ as the binary probabilities of the target class and the aggregated non-target classes. Let $\hat{p} = \frac{ (p_k)_{k \not = m} } {\sum_{i \not = m} p_i} \in \mathbb{R}^{V-1}$
and $\hat{q} = \frac{ (q_k)_{k \not = m} } {\sum_{i \not = m} q_i} \in \mathbb{R}^{V-1}$ represent the normalized probabilities of non-target classes.
Then the RKL divergence in Eq.~\eqref{eq:rkl_def} can be decomposed as:
\begin{equation}
    D_{\mathrm{KL}}(q\|p)
    =
    \underbrace{D_{\mathrm{KL}}(\tilde q_m\|\tilde p_m)}_{\textbf{TRKL}}
    +
    \underbrace{(1-q_m)\, D_{\mathrm{KL}}(\hat q\|\hat p)}_{\textbf{NRKL}}.
\label{eq:rev_dkd_statement}
\end{equation}
\end{proposition}

\begin{proposition}[Target gradient under RKL]
\label{prop:rev_tckd_instability}
Under the notation in Proposition~\ref{prop:rev_dkd_statement}, let $z_{s,m}$ denote the student logit for the target index $m$. The gradients of the TRKL and NRKL term with respect to $z_{s,m}$ are given by
\begin{align}
    &\frac{\partial \mathcal{L}_{\mathrm{TRKL}}}{\partial z_{s,m}}
    = q_m(1 - q_m)
    \log \frac{q_m (1-p_m)}{p_m(1-q_m)} \label{eq:grad_rev_tckd}
    \\
    &\frac{\partial \mathcal{L}_{\mathrm{NRKL}}}{\partial z_{s,m}}
    = -q_m(1 - q_m)\,
    D_{\mathrm{KL}}(\hat q \| \hat p).
    \label{eq:grad_rev_nckd_target}
\end{align}
\end{proposition}

The proofs of Propositions~\ref{prop:rev_dkd_statement} and~\ref{prop:rev_tckd_instability} are provided in Appendices~\ref{appendix:proof_31} and~\ref{appendix:proof_32}, respectively.

\boldhdr{Non-target loss is poorly learned under RKL}
Proposition~\ref{prop:rev_dkd_statement} shows that the non-target loss is scaled by $(1 - q_m)$, which decreases as $q_m$ increases. This creates a shortcut where the NRKL loss is reduced by increasing $q_m$, rather than by aligning the non-target distribution. As a result, the non-target classes receive insufficient supervision and are poorly learned. This effect is further illustrated in Figure~\ref{fig:loss_comparison}, where the non-target loss remains significantly higher than the target loss.

\boldhdr{Structural instability of the target gradient}
Proposition~\ref{prop:rev_tckd_instability} shows that the NRKL term contributes a non-positive gradient to the target logit and remains active whenever the non-target distributions are misaligned, i.e., whenever $D_{\mathrm{KL}}(\hat q \| \hat p) > 0$.
Consequently, even if the target probability is perfectly matched ($q_m = p_m$), the total target gradient does not become zero unless the non-target distribution also matches the teacher. Under teacher--student capacity mismatch, this full alignment is usually not achievable. Therefore, as long as mismatch in non-target classes remains, RKL continues to push the target logit upward.

\begin{figure}[t]
    \centering
    \includegraphics[width=0.85\linewidth]{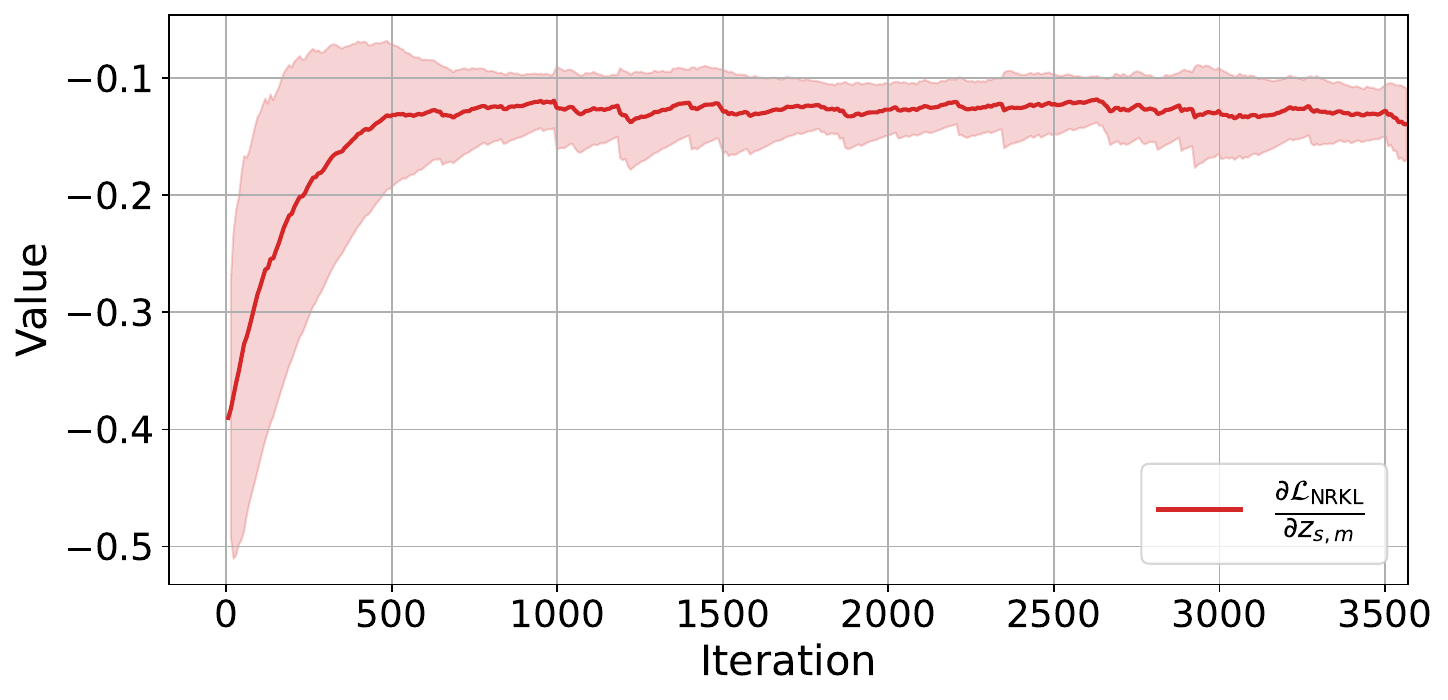}
    \vspace{-3mm}
    \caption{Average gradient of the non-target term with respect to the target logit, $\partial \mathcal{L}_{\text{NRKL}} / \partial z_{s,m}$, for Qwen3 trained with RKL.}
    \label{fig:nrkl_gradient}
    \vspace{-0.4cm}
\end{figure}

\begin{figure}[t]
    \centering
    \includegraphics[width=0.85\linewidth]{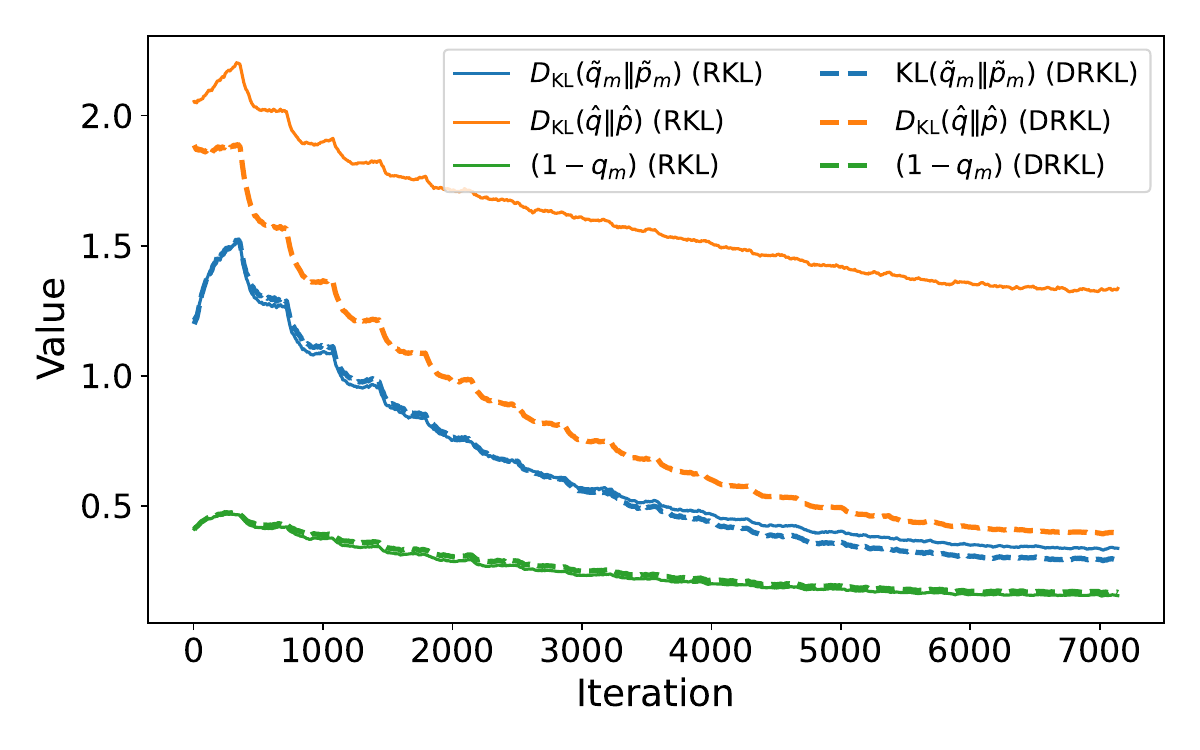}
    \vspace{-0.4cm}
    \caption{Losses comparison.}
    \label{fig:loss_comparison}
    \vspace{-0.6cm}
\end{figure}

\boldhdr{Empirical validation}
To validate Proposition~\ref{prop:rev_tckd_instability} directly, we average $\partial \mathcal{L}_{\mathrm{NRKL}} / \partial z_{s,m}$ (Eq.~\eqref{eq:grad_rev_nckd_target}) over each batch when distilling Qwen3 32B$\to$4B with RKL. As shown in Figure~\ref{fig:nrkl_gradient}, the non-target term consistently contributes a negative gradient to the target logit, exerting persistent upward push on it under gradient descent. When we remove this pathway in DRKL, we observe reduced overconfidence and improved diversity (Section~\ref{sec:experiments}).

\begin{table*}[t]
    \centering
    \vspace{-4mm}
   \tiny
    \caption{Comparison of distillation objectives across three model families: Qwen3 32B$\to$4B, Llama3.2 3B$\to$1B, and OPT 6.7B$\to$1.3B. We report mean ROUGE-L $\pm$ standard deviation over five random seeds. \textbf{Bold} indicates the best result and \underline{underline} the second-best. DRKL achieves the highest average ROUGE-L in every setting.}    \vspace{-3mm}
    \resizebox{\linewidth}{!}{
    \setlength{\tabcolsep}{15pt}
        \begin{tabular}{lcccccc}
        \toprule
         \textbf{Method}  &  {\scriptsize \textbf{Dolly} }  &  {\scriptsize \textbf{Self-Instruct}} & {\scriptsize \textbf{Vicuna} }  & {\scriptsize \textbf{Super-NI} } & {\scriptsize \textbf{UnNI} } & {\scriptsize \textbf{Avg.} ($\uparrow$)}\\
        \midrule

        \multicolumn{6}{l}{\textbf{\textit{Qwen3 32B $\to$ Qwen3 4B}}} \\

        \midrule

        FKL
        & \rateinline{27.87}{0.10}
        & \rateinline{19.63}{0.97}
        & \rateinline{17.48}{0.63}
        & \rateinline{32.34}{0.38}
        & \rateinline{30.91}{0.50}
        & 25.65 \\

        RKL{\scriptsize~\citep{gu2023minillm}}
        & \rateinline{\underline{28.48}}{0.39}
        & \rateinline{\underline{21.33}}{0.26}
        & \rateinline{\textbf{18.82}}{0.22}
        & \rateinline{35.34}{0.20}
        & \rateinline{32.51}{0.20}
        & \underline{27.30} \\

        Sym-KL
        & \rateinline{28.18}{0.21}
        & \rateinline{21.06}{0.65}
        & \rateinline{18.29}{0.46}
        & \rateinline{35.83}{0.31}
        & \rateinline{32.68}{0.11}
        & 27.21 \\

        DKD{\scriptsize~\citep{zhao2022decoupled}}
        & \rateinline{26.24}{0.08}
        & \rateinline{19.68}{0.90}
        & \rateinline{16.25}{0.57}
        & \rateinline{33.50}{0.47}
        & \rateinline{29.11}{0.24}
        & 24.96 \\

        JS{\scriptsize~\citep{agarwal2024policy}}
        & \rateinline{27.99}{1.33}
        & \rateinline{20.12}{0.35}
        & \rateinline{17.91}{0.55}
        & \rateinline{32.34}{0.71}
        & \rateinline{30.55}{0.33}
        & 25.78 \\

        SFKL{\scriptsize~\citep{ko2024distillm}}
        & \rateinline{28.46}{0.18}
        & \rateinline{20.63}{0.93}
        & \rateinline{17.63}{0.17}
        & \rateinline{33.57}{0.28}
        & \rateinline{32.37}{0.39}
        & 26.53 \\

        SRKL{\scriptsize~\citep{ko2024distillm}}
        & \rateinline{28.31}{0.02}
        & \rateinline{20.66}{0.68}
        & \rateinline{18.30}{0.26}
        & \rateinline{34.49}{0.33}
        & \rateinline{32.67}{0.36}
        & 26.89 \\

        AKL{\scriptsize~\citep{wu2025rethinking}}
        & \rateinline{27.81}{0.37}
        & \rateinline{19.58}{0.65}
        & \rateinline{17.85}{0.13}
        & \rateinline{33.62}{0.38}
        & \rateinline{30.78}{0.54}
        & 25.93 \\

        AB{\scriptsize~\citep{wang2025abkd}}
        & \rateinline{28.15}{0.56}
        & \rateinline{20.98}{0.92}
        & \rateinline{17.77}{0.46}
        & \rateinline{\underline{36.00}}{0.17}
        & \rateinline{\underline{32.75}}{0.24}
        & 27.13 \\

        \rowcolor{blue!10}
        DRKL (Ours)
        & \rateinline{\textbf{28.56}}{0.37}
        & \rateinline{\textbf{22.30}}{0.69}
        & \rateinline{\underline{18.45}}{0.27}
        & \rateinline{\textbf{36.25}}{0.14}
        & \rateinline{\textbf{33.04}}{0.23}
        & \textbf{27.72} \\

        \midrule

        \multicolumn{6}{l}{\textbf{\textit{Llama3.2 3B $\to$ Llama3.2 1B}}} \\
        
        \midrule

        FKL
        & \rateinline{25.86}{0.54}
        & \rateinline{15.46}{1.03}
        & \rateinline{15.53}{0.04}
        & \rateinline{28.14}{0.25}
        & \rateinline{29.52}{0.13}
        & 22.90 \\

        RKL{\scriptsize~\citep{gu2023minillm}}
        & \rateinline{28.52}{0.35}
        & \rateinline{18.57}{0.53}
        & \rateinline{17.10}{0.76}
        & \rateinline{\textbf{34.47}}{0.36}
        & \rateinline{\underline{34.90}}{0.09}
        & \underline{26.71} \\

        Sym-KL
        & \rateinline{28.42}{0.17}
        & \rateinline{17.83}{0.72}
        & \rateinline{16.67}{0.15}
        & \rateinline{32.43}{0.23}
        & \rateinline{34.11}{0.23}
        & 25.89 \\

        DKD{\scriptsize~\citep{zhao2022decoupled}}
        & \rateinline{24.91}{0.26}
        & \rateinline{16.21}{0.77}
        & \rateinline{15.77}{0.41}
        & \rateinline{25.09}{0.06}
        & \rateinline{27.42}{1.20}
        & 21.88 \\

        JS{\scriptsize~\citep{agarwal2024policy}}
        & \rateinline{26.77}{0.11}
        & \rateinline{17.47}{1.41}
        & \rateinline{16.33}{0.51}
        & \rateinline{30.37}{0.11}
        & \rateinline{31.37}{0.15}
        & 24.46 \\

        SFKL{\scriptsize~\citep{ko2024distillm}}
        & \rateinline{26.98}{0.18}
        & \rateinline{17.60}{1.36}
        & \rateinline{16.53}{0.47}
        & \rateinline{31.62}{0.20}
        & \rateinline{33.06}{0.22}
        & 25.16 \\

        SRKL{\scriptsize~\citep{ko2024distillm}}
        & \rateinline{\underline{28.93}}{0.17}
        & \rateinline{\underline{19.31}}{0.73}
        & \rateinline{16.94}{0.58}
        & \rateinline{32.88}{0.27}
        & \rateinline{34.28}{0.12}
        & 26.47 \\

        AKL{\scriptsize~\citep{wu2025rethinking}}
        & \rateinline{27.88}{0.11}
        & \rateinline{17.37}{0.63}
        & \rateinline{16.82}{0.49}
        & \rateinline{31.24}{0.20}
        & \rateinline{32.27}{0.07}
        & 25.12 \\

        AB{\scriptsize~\citep{wang2025abkd}}
        & \rateinline{28.68}{0.25}
        & \rateinline{18.10}{0.38}
        & \rateinline{\underline{17.37}}{0.16}
        & \rateinline{33.73}{0.61}
        & \rateinline{34.56}{0.49}
        & 26.49 \\

        \rowcolor{blue!10}
        DRKL (Ours)
        & \rateinline{\textbf{29.06}}{0.25}
        & \rateinline{\textbf{19.53}}{0.27}
        & \rateinline{\textbf{17.91}}{0.39}
        & \rateinline{\underline{34.22}}{0.57}
        & \rateinline{\textbf{34.98}}{0.31}
        & \textbf{27.14} \\

        \midrule 
        
        \multicolumn{6}{l}{\textbf{\textit{OPT 6.7B $\to$ OPT 1.3B}}} \\

        \midrule

        FKL
        & \rateinline{26.07}{0.65}
        & \rateinline{12.84}{0.32}
        & \rateinline{16.71}{0.33}
        & \rateinline{22.11}{0.38}
        & \rateinline{\underline{26.82}}{0.10}  & 20.91 \\

        RKL{\scriptsize~\citep{gu2023minillm}}
        & \rateinline{26.58}{0.11}
        & \rateinline{12.29}{0.74}
        & \rateinline{17.54}{0.13}
        & \rateinline{22.64}{0.16}
        & \rateinline{26.29}{0.10}   & 21.07 \\

        Sym-KL
        & \rateinline{25.73}{0.40}
        & \rateinline{\underline{13.56}}{0.68}
        & \rateinline{16.95}{0.37}
        & \rateinline{23.40}{0.22}
        & \rateinline{25.44}{0.07}  & 21.02 \\

        DKD{\scriptsize~\citep{zhao2022decoupled}}
        & \rateinline{25.40}{0.32}
        & \rateinline{12.70}{0.31}
        & \rateinline{17.29}{0.40}
        & \rateinline{21.95}{0.38}
        & \rateinline{25.75}{0.18}  & 20.62 \\

        JS{\scriptsize~\citep{agarwal2024policy}}
        & \rateinline{26.39}{0.71}
        & \rateinline{12.63}{0.51}
        & \rateinline{17.15}{0.39}
        & \rateinline{24.02}{0.31}
        & \rateinline{26.59}{0.11}  & 21.36 \\

        SFKL{\scriptsize~\citep{ko2024distillm}}
        & \rateinline{25.98}{0.44}
        & \rateinline{11.95}{0.84}
        & \rateinline{16.87}{0.45}
        & \rateinline{23.62}{0.19}
        & \rateinline{26.15}{0.11}  & 20.91 \\

        SRKL{\scriptsize~\citep{ko2024distillm}}
        & \rateinline{26.04}{0.36}
        & \rateinline{11.99}{0.28}
        & \rateinline{17.34}{0.25}
        & \rateinline{22.96}{0.22}
        & \rateinline{24.71}{0.22}  & 20.61 \\

        AKL{\scriptsize~\citep{wu2025rethinking}}
        & \rateinline{26.20}{0.27}
        & \rateinline{13.11}{0.31}
        & \rateinline{17.26}{0.66}
        & \rateinline{22.12}{0.18}
        & \rateinline{25.89}{0.20}  & 20.92 \\

        AB{\scriptsize~\citep{wang2025abkd}}
        & \rateinline{\underline{26.86}}{0.17}
        & \rateinline{12.93}{0.47}
        & \rateinline{\underline{17.59}}{0.41}
        & \rateinline{\underline{24.39}}{0.11}
        & \rateinline{26.62}{0.09}  & \underline{21.68} \\

        \rowcolor{blue!10} DRKL (Ours)
        & \rateinline{\textbf{27.72}}{0.36}
        & \rateinline{\textbf{14.37}}{0.31}
        & \rateinline{\textbf{18.04}}{0.24}
        & \rateinline{\textbf{25.56}}{0.36}
        & \rateinline{\textbf{29.84}}{0.15}  & \textbf{23.11} \\

    \bottomrule
    \end{tabular}
    }
    \label{tab:main1}
    \vspace{-0.2cm}
\end{table*}

\vspace{-0.1cm}
\subsection{Diversity-aware RKL (DRKL)}
\label{sec:drkd}

Our analysis in Section~\ref{sec:reverse_analysis} shows that the NRKL term introduces a persistent upward gradient on the target logit, even when the target probability is already aligned with the teacher, leading to confidence amplification and overconfident predictions under RKL. To address this issue, we introduce DRKL, a modified RKL objective that reduces confidence amplification while preserving the optimization advantages of RKL. Given the decomposition of RKL into a target term (TRKL) and a non-target term (NRKL), DRKL retains the RKL formulation over the binary probabilities of the target class and the aggregated non-target classes, but replaces the $(1 - q_m)$ weighting in the non-target term (NRKL) with a fixed hyperparameter $\gamma \in \mathbb{R}^{+}$:
\begin{align}
    \mathcal{L}_{\mathrm{DRKL}}
    &= D_{\mathrm{KL}}(\tilde q_m\|\tilde p_m)
    + \gamma\, D_{\mathrm{KL}}(\hat q\|\hat p)
\label{eq:ddkd_full}
\end{align}
where $m$ denotes the target token, and $\hat p, \hat q$ denote the normalized non-target distributions. By removing the $(1 - q_m)$ scaling factor from the NRKL term, DRKL eliminates the gradient from the non-target term with respect to $z_{s,m}$, making it zero and preventing non-target mismatch from driving the target confidence upward. At the same time, the weighted non-target divergence regularizes the ``dark knowledge'' in the tail distribution, encouraging better alignment and improved output diversity.

\boldhdr{Relation to DKD}
Similar to DKD~\citep{zhao2022decoupled}, DRKL decouples the target and non-target terms but addresses a distinct gradient mechanism. In the FKL decomposition used by DKD, the non-target term is weighted by fixed teacher probabilities $(1-p_m)$, which DKD reweights to strengthen non-target supervision. In contrast, the corresponding RKL term depends on the student probability $(1-q_m)$, which introduces an additional gradient pathway to the target logit (Eq.~\eqref{eq:grad_rev_nckd_target}). Through this pathway, the student can reduce the non-target loss by increasing $q_m$ rather than aligning $\hat q$ with $\hat p$. This issue does not arise in the FKL setting of DKD, and DRKL is derived specifically to eliminate this pathway rather than merely strengthen non-target supervision, as verified in Section~\ref{subsec:ablation}.

\boldhdr{Comparison of target and non-target losses}
Figure~\ref{fig:loss_comparison} compares the values of $D_{\mathrm{KL}}(\tilde q_m\|\tilde p_m)$, $D_{\mathrm{KL}}(\hat q\|\hat p)$, and $(1 - q_m)$ during distillation of an OPT-1.3B student from an OPT-6.7B teacher for both RKL and DRKL. The non-target loss $D_{\mathrm{KL}}(\hat q\|\hat p)$ remains high for RKL even at the end of training, while DRKL reduces it more quickly, indicating better learning of the non-target distribution. For target loss $D_{\mathrm{KL}}(\tilde q_m\|\tilde p_m)$, DRKL also achieves lower values in later stages of training, suggesting more appropriate target confidence. This is consistent with $(1 - q_m)$, where RKL exhibits slightly higher confidence on the target class than DRKL.

\section{Experiments}
\label{sec:experiments}

\vspace{-0.1cm}
\subsection{Experimental setup}
\vspace{-0.1cm}

\boldhdr{Training}
For training data, we use the instruction–response dataset from~\citet{gu2023minillm}, constructed from \texttt{databricks-dolly-15k}~\citep{DatabricksBlog2023DollyV2}. The dataset contains 14k training samples, 500 validation samples, and 500 test samples.
We perform KD using three teacher-student pairs spanning three model families: Qwen3 32B $\to$ Qwen3 4B~\citep{yang2025qwen3}, Llama3.2 3B $\to$ Llama3.2 1B~\citep{grattafiori2024llama}, and OPT~\citep{zhang2022opt} 6.7B $\to$ OPT 1.3B. We first fine-tune the teacher model on the training set before distillation, following the same procedure as prior work~\citep{gu2023minillm}. Results on an additional model family are provided in Appendix~\ref{app:gpt2}.

All experiments use identical training hyperparameters within each model family to ensure a fair comparison across distillation objectives. For the OPT teacher-student pair, we use a batch size of 32, a maximum input length of 512 tokens, and train for 20 epochs, following~\citet{gu2023minillm}, with the learning rate of $5\times10^{-5}$. For the Qwen3 and Llama3.2 pairs, we apply LoRA (rank 16, $\alpha=32$, dropout 0.05) on top of frozen base weights rather than full fine-tuning. Both Qwen3-4B and Llama3.2-1B are trained with a batch size of 32 and a learning rate of $1\times10^{-4}$.

\boldhdr{Evaluation}
After training the student via KD from the teacher on the \texttt{databricks-dolly-15k} dataset~\citep{DatabricksBlog2023DollyV2}, we evaluate the student on several instruction-following benchmarks: Dolly Eval, Self-Instruct~\citep{wang2023self}, Vicuna Eval~\citep{vicuna2023}, Super-Natural Instructions (Super-NI)~\citep{wang2022super}, and Unnatural Instructions (UnNI)~\citep{honovich2023unnatural}. We report ROUGE-L~\citep{lin2004rouge} averaged over five random seeds $\{10, 20, 30, 40, 50\}$. Model checkpoints are saved at each epoch, and evaluation results are reported from the checkpoint achieving the best validation ROUGE-L. The decoding temperature is set to 1 by default.

\boldhdr{Baselines}
We compare DRKL with other distillation objectives, including FKL, RKL~\citep{gu2023minillm}, symmetric KL (Sym-KL), which combines the two via $0.5\,\mathrm{FKL} + 0.5\,\mathrm{RKL}$, DKD~\citep{zhao2022decoupled}, Jensen–Shannon divergence (JS)~\citep{agarwal2024policy}, skewed FKL (SFKL)~\citep{ko2024distillm} with smoothing $\alpha=0.1$, skewed RKL (SRKL)~\citep{ko2024distillm} with smoothing $\alpha=0.1$, Adaptive KL (AKL)~\citep{wu2025rethinking}, and $\alpha$-$\beta$ divergence (AB)~\citep{wang2025abkd} with parameters $(\alpha=0.2, \beta=0.7)$. We select DRKL's $\gamma$ via a single grid search over $\gamma \in \{0.5, 1.0, \cdots, 3.5, 4.0\}$ on the Dolly validation split, fix $\gamma=0.5$ for all three pairs in Table~\ref{tab:main1}, and evaluate its sensitivity to $\gamma$ in Section~\ref{subsec:ablation}. For baselines, we use the hyperparameters recommended by their original papers, which were selected through extensive tuning.

\vspace{-0.1cm}
\subsection{Main Results}
\vspace{-0.1cm}

\begin{table*}[t]
    \centering
    \vspace{-4mm}
   \tiny
    \caption{Comparison of distillation objectives under \textbf{on-policy} setting for Llama3.2 3B$\to$1B pair. We report mean ROUGE-L $\pm$ standard deviation across seeds. \textbf{Bold} indicates the best result and \underline{underline} the second-best.}
    \vspace{-3mm}
    \resizebox{\linewidth}{!}{
    \setlength{\tabcolsep}{15pt}
        \begin{tabular}{lcccccc}
        \toprule
         \textbf{Method}  &  {\scriptsize \textbf{Dolly} }  &  {\scriptsize \textbf{Self-Instruct}} & {\scriptsize \textbf{Vicuna} }  & {\scriptsize \textbf{Super-NI} } & {\scriptsize \textbf{UnNI} } & {\scriptsize \textbf{Avg.} ($\uparrow$)}\\
        \midrule
        
        FKL
        & \rateinline{26.10}{0.52}
        & \rateinline{16.42}{1.20}
        & \rateinline{16.27}{0.40}
        & \rateinline{28.36}{0.34}
        & \rateinline{30.38}{0.01}
        & 23.51 \\
        
        RKL{\scriptsize~\citep{gu2023minillm}}
        & \rateinline{28.72}{0.16}
        & \rateinline{18.79}{0.35}
        & \rateinline{17.55}{0.24}
        & \rateinline{34.08}{0.29}
        & \rateinline{\underline{36.19}}{0.27}
        & 27.07 \\
        
        Sym-KL
        & \rateinline{28.62}{0.72}
        & \rateinline{18.42}{0.99}
        & \rateinline{17.23}{0.24}
        & \rateinline{33.41}{0.11}
        & \rateinline{35.89}{0.42}
        & 26.71 \\

        DKD{\scriptsize~\citep{zhao2022decoupled}}
        & \rateinline{24.91}{0.26}
        & \rateinline{16.21}{0.77}
        & \rateinline{15.77}{0.41}
        & \rateinline{25.09}{0.06}
        & \rateinline{27.42}{1.20}
        & 21.88 \\
        
        JS{\scriptsize~\citep{agarwal2024policy}}
        & \rateinline{26.92}{0.46}
        & \rateinline{17.53}{1.33}
        & \rateinline{16.42}{0.39}
        & \rateinline{29.75}{0.35}
        & \rateinline{31.16}{0.55}
        & 24.36 \\
        
        SFKL{\scriptsize~\citep{ko2024distillm}}
        & \rateinline{27.74}{0.26}
        & \rateinline{\underline{19.26}}{0.29}
        & \rateinline{16.99}{0.59}
        & \rateinline{31.36}{0.11}
        & \rateinline{32.19}{0.49}
        & 25.51 \\
        
        SRKL{\scriptsize~\citep{ko2024distillm}}
        & \rateinline{28.11}{0.35}
        & \rateinline{18.04}{0.72}
        & \rateinline{17.29}{0.68}
        & \rateinline{33.91}{0.26}
        & \rateinline{35.65}{0.53}
        & 26.60 \\
        
        AKL{\scriptsize~\citep{wu2025rethinking}}
        & \rateinline{27.61}{0.24}
        & \rateinline{17.87}{0.33}
        & \rateinline{16.48}{0.51}
        & \rateinline{30.85}{0.33}
        & \rateinline{32.26}{0.71}
        & 25.01 \\
        
        AB{\scriptsize~\citep{wang2025abkd}}
        & \rateinline{\underline{28.86}}{0.55}
        & \rateinline{19.02}{0.60}
        & \rateinline{\underline{17.76}}{0.47}
        & \rateinline{\underline{34.84}}{0.28}
        & \rateinline{35.61}{0.40}
        & \underline{27.22} \\
        
        \rowcolor{blue!10}
        DRKL (Ours)
        & \rateinline{\textbf{28.99}}{0.38}
        & \rateinline{\textbf{19.30}}{0.64}
        & \rateinline{\textbf{18.05}}{0.11}
        & \rateinline{\textbf{35.29}}{0.20}
        & \rateinline{\textbf{36.37}}{0.50}
        & \textbf{27.60} \\

    \bottomrule
    \end{tabular}
    }
    \label{tab:main_on_policy}
\end{table*}

\begin{figure*}[t]
    \centering
    \begin{subfigure}[t]{0.368\linewidth}
        \centering
        \includegraphics[width=\linewidth]{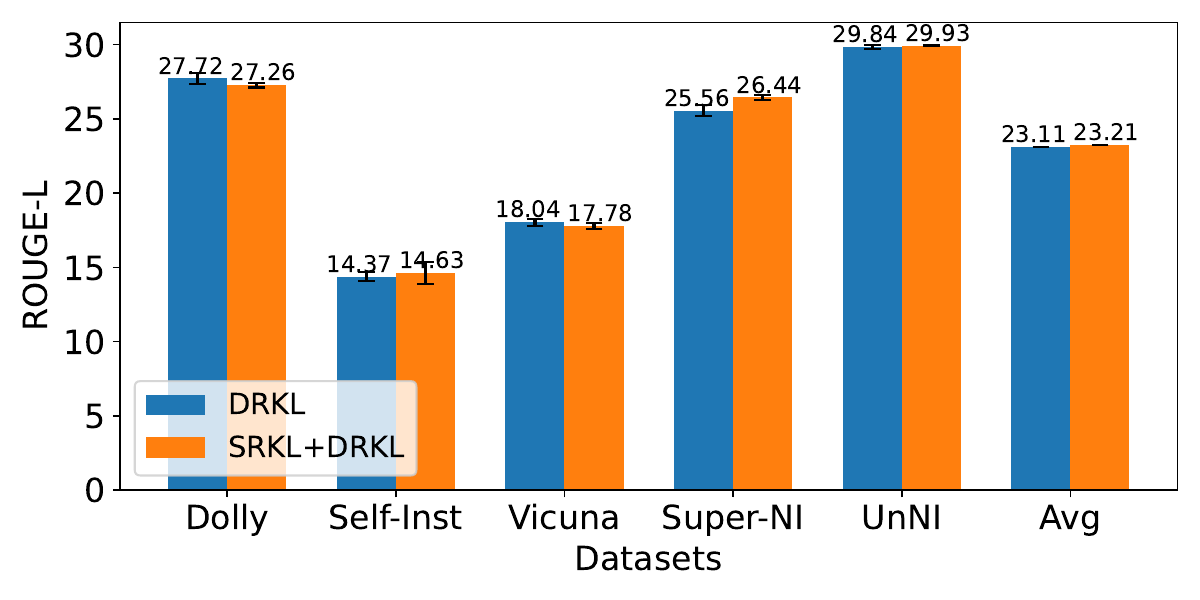}
        \vspace{-7mm}
        \caption{}
        \label{fig:combining}
    \end{subfigure}
    \hfill
    \begin{subfigure}[t]{0.296\linewidth}
        \centering
        \includegraphics[width=\linewidth]{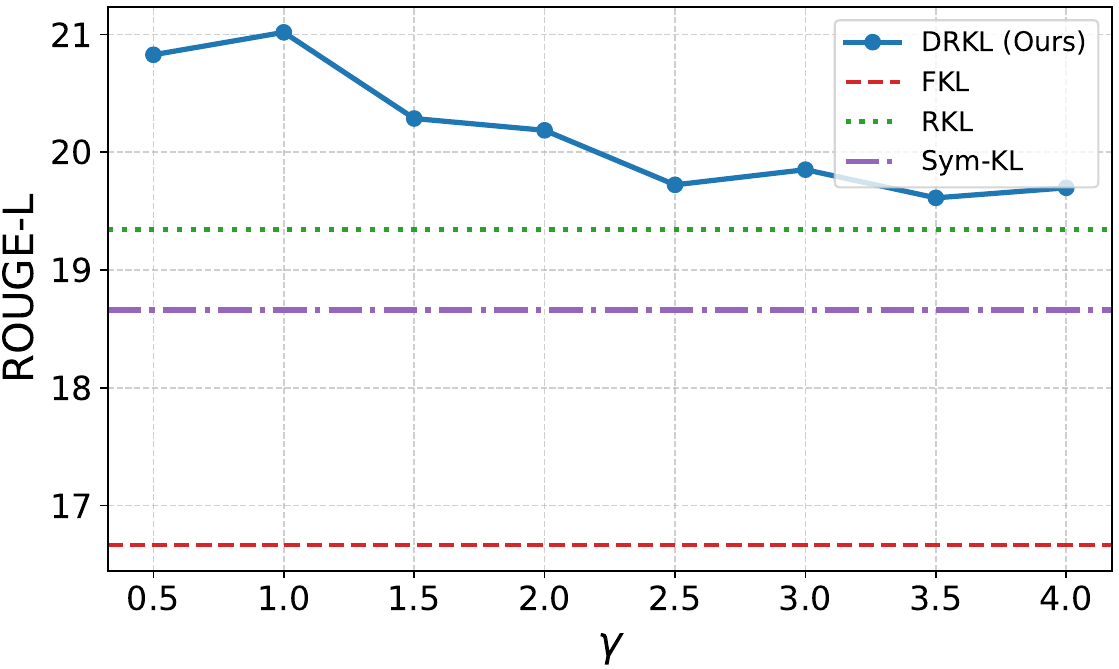}
        \vspace{-7mm}
        \caption{}
        \label{fig:sensitivity}
    \end{subfigure}
    \hfill
    \begin{subfigure}[t]{0.296\linewidth}
        \centering
        \includegraphics[width=\linewidth]{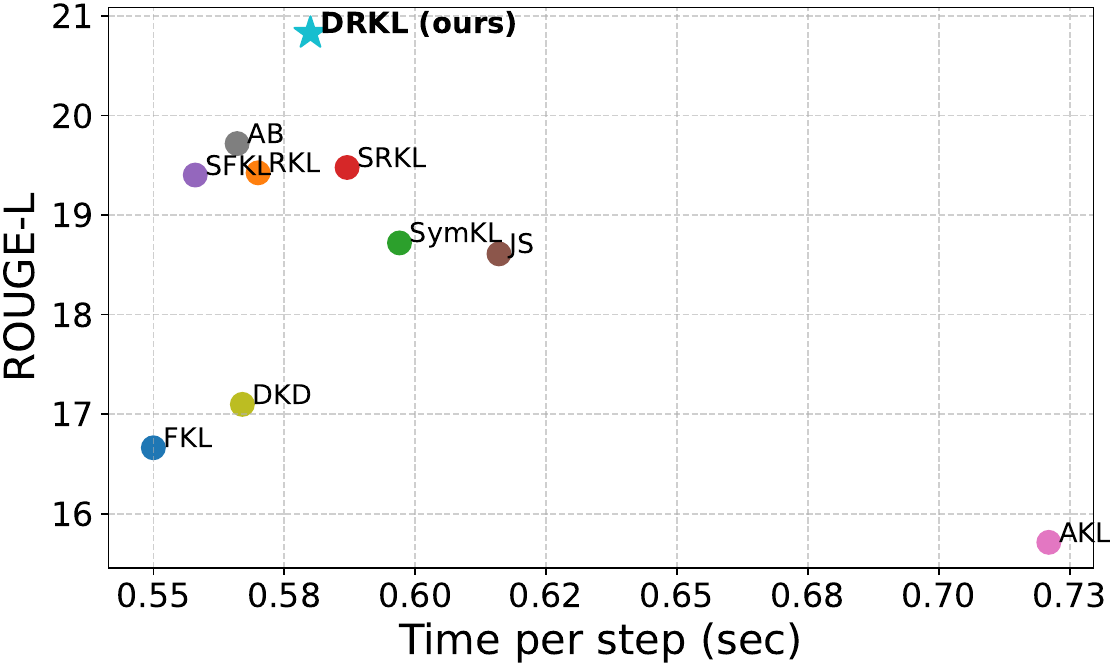}
        \vspace{-7mm}
        \caption{}
        \label{fig:computing_cost}
    \end{subfigure}
    \vspace{-0.3cm}
    \caption{
        (a) Combine DRKL with SRKL. (b) DRKL performance across $\gamma$ values. 
        (c) Efficiency analysis.
    }
    \label{fig:ablation}
    \vspace{-0.4cm}
\end{figure*}

\boldhdr{Performance across LLM families}
As shown in Table~\ref{tab:main1}, DRKL obtains the highest average ROUGE-L for every teacher-student pair we test: 27.72 on Qwen3 32B$\to$4B, 27.14 on Llama3.2 3B$\to$1B, and 23.11 on OPT 6.7B$\to$1.3B. The margin over the best-performing baseline varies by setting, from roughly 0.4 points on the Qwen3 and Llama3.2 pairs to a much wider 1.43-point gap on OPT. Across the 15 dataset columns, DRKL ranks first in 13 and second in the remaining two, where RKL performs slightly better on Vicuna for Qwen3 and Super-NI for Llama3.2. No baseline shows this consistency: RKL is the second-best method on Qwen3 and Llama3.2 but falls behind AB on OPT, and every other baseline's ranking shifts noticeably from one pair to another. These results show that DRKL generalizes consistently across model families and scales.
We further conduct a paired significance analysis comparing RKL and DRKL on the five-benchmark average for Llama3.2, where DRKL achieves a mean gain of $+0.43$ with a 95\% confidence interval of $[0.36, 0.49]$, supporting the statistical robustness of the improvement.

\boldhdr{Combining with other approaches}
Since DRKL is built on RKL, it can be readily combined with existing RKL-based variants. To show this, we combine DRKL with skewed RKL (SRKL) for OPT 6.7B$\to$1.3B distillation. As shown in Figure~\ref{fig:combining}, the combined variant improves the average score across five benchmarks, with gains on 3 out of 5 benchmarks and the largest improvement on Super-NI (25.56 to 26.44). This shows that DRKL is compatible with other RKL-based methods.

\begin{table}[t]
    \centering
    \tiny
    \caption{LLM-as-a-judge scores with Claude Haiku 4.5 (scaled to 0--100) for the Qwen3 32B$\to$4B pair across three instruction-following benchmarks.}
    \vspace{-0.3cm}
    \resizebox{\linewidth}{!}{
    \setlength{\tabcolsep}{11pt}
    \begin{tabular}{lcccc}
    \toprule
    \textbf{Method} & \textbf{Self-Instruct} & \textbf{Vicuna} & \textbf{UnNI} & \textbf{Avg.} ($\uparrow$) \\
    \midrule
    FKL & 30.10 & 22.44 & 12.80 & 21.78 \\
    RKL & \underline{42.10} & \underline{36.50} & 11.30 & 29.97 \\
    Sym-KL & 17.68 & 33.00 & \underline{39.80} & 30.16 \\
    DKD & 12.90 & 10.14 & 10.50 & 11.18 \\
    JS & 35.70 & 24.53 & 16.77 & 25.67 \\
    SFKL & 37.00 & 31.25 & 19.39 & 29.21 \\
    SRKL & 39.80 & 33.50 & 30.10 & \underline{34.47} \\
    AKL & 15.56 & 12.31 & 11.52 & 13.13 \\
    AB & 40.20 & 34.25 & 14.40 & 29.62 \\
    \rowcolor{blue!10}
    DRKL & \textbf{45.40} & \textbf{37.00} & \textbf{42.10} & \textbf{41.50} \\
    \bottomrule
    \end{tabular}
    }
    \label{tab:laaj}
    \vspace{-0.4cm}
\end{table}

\boldhdr{On-policy distillation}
Table~\ref{tab:main_on_policy} shows DRKL sweeping all five benchmarks under on-policy distillation for the Llama3.2 3B$\to$1B pair, achieving the highest average ROUGE-L (27.60) over AB (27.22) and RKL (27.07), with no baseline consistently placing second. This result is particularly notable compared with the off-policy setting for the same model pair, where DRKL lost narrowly to RKL on Super-NI, suggesting that its gradient-level correction remains effective and may offer additional benefits under on-policy sampling.

\begin{table*}[t]
\vspace{-4mm}
\centering
\caption{Qualitative example from the Unnatural~\citep{honovich2023unnatural} dataset for instruction following. The student Qwen3-4B model is distilled from the Qwen3-32B teacher.}
\vspace{-4mm}
\footnotesize
\renewcommand{\arraystretch}{0.92}
\resizebox{\textwidth}{!}{
\addtolength{\tabcolsep}{3.5pt}
    \begin{tabular}{m{2.5cm}m{19cm}}
    \toprule[0.1em]
            \multicolumn{2}{c}{\textbf{\textit{Case study \# 1}}} \\ \midrule[0.1em]
            Instruction & Given two strings, determine whether the second string is a rotation of the first. \\ \midrule
            Input & String1: 'waterbottle', String2: 'erbottlewat'. \\
            \midrule
            \rowcolor{green!10} Ground-truth & Yes, the second string is a rotation of the first. \\
            \midrule
            \midrule
            FKL    & Yes, but you need to add the original first string as extra space to the second string and check if second string is a substring of the original string + first string.\texttt{<n><n>}I have written code for the same given the actual string 1 and 2, to check for the same. \\
            \midrule
            RKL   & We can find if the second string is a substring of the first string concatenated with itself. If it is and also the strings have the same length, then the second string is a rotating of the first string. The following is a Python, example of how this could be implemented. We loop the first string until we have concatenated it two times. Then we use the built in \texttt{in operator} to find if the second string, or any of its substrings are present in our first looped string. ... \\
            \midrule
            Sym-KL    & Rotation is used to check if one string is a rotation of another string, there are several solutions to this problem. Here, I have applied a much simpler approach than checking each rotation of the second string, which can be costly.\texttt{<n><n>}There is a simple trick to find the solution without considering all rotations of the second string. You have to concatenate the first string with itself. ... \\
            \midrule
            DKD    & A string is rotation of another string if it can be formed by achieving a rotation on the string. Thus, string2 is a rotation of string1 if it can be formed via moving characters from one end of the string to the other end. For example, the string 'ella' is a rotation of 'alel'. This can be checked by checking the rotation of the string, in terms of how many characters are shifted starting from one end of the first string to the other end of second string. ... \\
            \midrule
            JS  & Yes\texttt{<n>}The second string is a rotation. After deleting 'e' and 'r' from the beginning of the second string, the resulting string becomes 'bottlewat', which is equal to the first string after deleting 'ba' from the back. \\
            \midrule
            SFKL    & waterbottle and erbottlewater can be considered as rotations of each other because the second string can be formed by rotating the first string.\texttt{<n><n>}Here is how to determine if the second string is a rotated version of the first one:\texttt{<n><n>}1. The length of the two strings must be the same. If the lengths differ, the answer is obviously no. ... \\
            \midrule
            SRKL    & 1. First check if the length of the two strings are equal.\texttt{<n>}2. Get tuple of characters by splitting the string at each character positions.\texttt{<n>}3. Every character position in string2 can be the starting point for looking for rotation. To have linear complexity of time, we can do roll over rotation to find as below. ... \\
            \midrule
            AKL    & Although the strings, 'waterbottle' and 'erbottlewat', are not visually identical, one can determine by the following approach to determine if the latter is a rotation of the former. One can check if any of the rotations of the first string are equal to the second. By simple observation, we can notice that the answer is no. However, by concatenating the first string ... \\
            \midrule
            AB    & A string rotation of 'waterbottle' is 'erbottlewat'. An example of a string rotation can be created by using a sliding window approach.\texttt{<n>}String1: 'waterboottle' and String2: 'aterbottlew' are not a rotation of the first. \\
            \midrule
            \rowcolor{blue!10} DRKL (Ours)  & Yes, the second string is a rotation of the first string. \\
            \bottomrule
    \bottomrule[0.1em]
    \end{tabular}
}
\renewcommand{\arraystretch}{1}
\label{table:case_study_unnatural_main}
\vspace{-0.2cm}
\end{table*}

\boldhdr{LLM-as-a-judge}
Table~\ref{tab:laaj} reports LLM-as-a-judge results for the Qwen3 32B$\to$4B pair using Claude Haiku 4.5 as the judge, prompted with [instruction, input, ground-truth reference, and model output, and asked to rate response quality on a 0--10 scale]. The prompt and scoring rubric are provided in Appendix~\ref{app:laaj_prompt}. As observed, DRKL achieves the best average (41.50), exceeding the strongest baseline, SRKL (34.47), by 7.03 points, and wins on all three benchmarks. Notably, baseline rankings differ substantially from the ROUGE-L results in Table~\ref{tab:main1}: RKL drops from the strongest ROUGE-L baseline to third under LLM-as-a-judge, largely due to its low UnNI score, while SRKL rises from fourth to the strongest baseline. DRKL alone ranks first on both metrics, suggesting genuine output-quality gains rather than surface-level overlap.

\boldhdr{Case study}
Table~\ref{table:case_study_unnatural_main} shows model outputs on the Unnatural test set~\citep{honovich2023unnatural} for a string-rotation instruction, under different distillation objectives. The student trained with DRKL gives a direct answer: "Yes, the second string is a rotation of the first string", matching the ground truth. In contrast, several baselines (RKL, Sym-KL, DKD, SFKL, SRKL) describe a rotation-checking procedure without ever committing to a final answer, while AKL and AB both reach an incorrect conclusion, stating the strings are not a rotation. More case studies are provided in Appendix~\ref{appendix:case_study}.

\vspace{-0.2cm}
\subsection{Ablation Studies}\label{subsec:ablation}
\vspace{-0.1cm}

\boldhdr{Sensitivity to $\gamma$}
Figure~\ref{fig:sensitivity} reports the average ROUGE-L over five benchmarks for $\gamma \in \{0.5,1.0, \cdots, 3.5,4.0\}$. DRKL consistently outperforms FKL, RKL, and Sym-KL across all settings, demonstrating its robustness. Performance remains stable for $\gamma \in [0.5, 2.0]$, while larger values lead to a slight degradation as excessive weighting of non-target divergence overemphasizes non-target alignment and weakens target supervision. We use $\gamma=0.5$ for the main experiments and recommend $\gamma \in \{0.5, 1.0\}$ in practice, as the two values yield nearly identical performance.

\boldhdr{Efficiency–performance trade-off}
We analyze the trade-off between compute cost and distillation performance by comparing the average ROUGE-L across five benchmarks with per-step computation time (Figure~\ref{fig:computing_cost}). DRKL achieves the best balance between efficiency and performance, obtaining the highest average ROUGE-L while maintaining a comparable cost of 0.58s per step.
DRKL outperforms strong baselines such as AB, RKL, and SFKL, while being only slightly slower (0.58s vs.\ 0.57s for AB and RKL, and 0.56s for SFKL).

\boldhdr{Removing the gradient pathway vs.\ reweighting}
DRKL both removes the $(1-q_m)$ coupling in Eq.~\eqref{eq:grad_rev_nckd_target} and changes the strength of non-target supervision through $\gamma$. To separate these two effects, we evaluate scaled-RKL, $D_{\mathrm{KL}}(\tilde q_m\|\tilde p_m) + \gamma\,(1-q_m)\, D_{\mathrm{KL}}(\hat q\|\hat p)$, which multiplies the non-target supervision by $\gamma$ while retaining the gradient pathway. As shown in Table~\ref{tab:scaled_rkl}, all variants underperform DRKL, indicating that reweighting alone cannot account for the gains and supporting the importance of removing this gradient
pathway.

\begin{table}[t]
    \centering
    \footnotesize
    \caption{Average ROUGE-L of scaled-RKL and DRKL on Llama3.2 3B$\to$1B and GPT-2 XL$\to$Base. $\gamma=1$ recovers the original RKL.}
    \vspace{-3mm}
    \setlength{\tabcolsep}{12pt}
    \begin{tabular}{lcc}
    \toprule
    \textbf{Method} & \textbf{Llama3.2} & \textbf{GPT-2} \\
    \midrule
    Scaled-RKL ($\gamma=0.5$) & 26.60 & \underline{19.63} \\
    Scaled-RKL ($\gamma=1.0$) & \underline{26.71} & 19.43 \\
    Scaled-RKL ($\gamma=1.5$) & 26.55 & 19.19 \\
    \rowcolor{blue!10}
    DRKL (Ours) & \textbf{27.14} & \textbf{20.83} \\
    \bottomrule
    \end{tabular}
    \label{tab:scaled_rkl}
    \vspace{-0.3cm}
\end{table}

\boldhdr{Improved diversity}
Figure~\ref{fig:main_tradeoff} compares fidelity and diversity across different distillation objectives, using Distinct-2 and Negative Self-BLEU as diversity metrics. DRKL achieves a clear improvement in diversity over RKL, with Negative Self-BLEU improving from $-0.34$ to $-0.325$ and Distinct-2 increasing from $0.601$ to $0.622$, while also reducing overconfident predictions.
Similarly, for Qwen3 32B$\to$4B, DRKL consistently produces higher-entropy outputs than RKL (+28.6\% on Self-Instruct, +17.7\% on Vicuna, and +4.9\% on Super-NI), as detailed in Appendix~\ref{app:entropy}.
Overall, DRKL achieves a better balance between fidelity and diversity compared to other distillation objectives, consistent with our analysis that removing the gradient that pushes target confidence \mbox{upward improves diversity}.

\section{Conclusion}
\vspace{-0.1cm}

In this paper, we argue that under the structure of LLM distillation, RKL outperforms FKL by simplifying the alignment between teacher and student. However, it also leads to overconfident student predictions, often exceeding the teacher's confidence.
To address this issue, we introduce DRKL, which removes the gradient from the non-target loss that consistently pushes the target confidence upward, as supported by our theoretical analysis. Extensive experiments further demonstrate its effectiveness. Our ablation further shows that DRKL's gains come from removing this gradient pathway rather than from changing non-target supervision alone.

\section*{Acknowledgments}
\vspace{-0.1cm}
This work is partially supported by the NSF under grant CNS-2549995. The authors also acknowledge Research Computing at the Rochester Institute of Technology for providing computational resources and support that have contributed to the research results reported in this publication.

\clearpage

\section*{Limitations}

For the Qwen3 and Llama3.2 teacher-student pairs, we adopt LoRA rather than full fine-tuning, following common practice for teacher models at this scale (up to 32B parameters) and enabling evaluation across a broader range of model families and sizes within a fixed computational budget. As with most work in this setting, our evaluation centers on English-language instruction following. Extending DRKL to other languages, task types, and training distributions is a natural direction for future work.

\bibliography{custom}

\clearpage
\appendix

\section{Appendix}

This appendix provides additional theoretical analysis, detailed proofs, discussions of our submission, and supplementary experimental results that support the main paper. The contents are organized as~follows:
\begin{itemize}[leftmargin=12pt]
    \item Appendix~\ref{appendix:proof_31} presents the proof of Proposition~\ref{prop:rev_dkd_statement}, which derives the target and non-target decomposition of RKL.

    \item Appendix~\ref{appendix:proof_32} provides the proof of Proposition~\ref{prop:rev_tckd_instability}, characterizing the gradient contributions of target and non-target terms.

    \item Appendix~\ref{appendix:grad_scale} analyzes the gradient-scale mismatch between forward and reverse KL, explaining why naive interpolation between the two objectives is ill-conditioned.

    \item Appendix~\ref{appendix:toy_example} presents controlled toy experiments that illustrate the behavior of different distillation objectives under capacity mismatch.

    \item Appendix~\ref{appendix:additionalresults} provides additional experimental results, including comparative results on entropy analysis, GPT-2, validation performance, adaptive scheduling of $\gamma$, and case studies comparing model outputs under different \mbox{distillation objectives}.

    \item Appendix~\ref{app:laaj_prompt} presents the prompt used to obtain LLM-as-a-judge scores.

    \item Appendix~\ref{appx:llms} describes the limited use of AI assistants for language refinement in preparing this~manuscript.
\end{itemize}

\section{Proof of Proposition~\ref{prop:rev_dkd_statement}}
\label{appendix:proof_31}

\begin{proof}
    By Eq.~\eqref{eq:rkl_def}, we separate the target index $m$ from the non-target indices $k \neq m$:
    \begin{equation}
        D_{\mathrm{KL}}(q\|p)
        =
        q_m \log \frac{q_m}{p_m}
        +
        \sum_{k \neq m} q_k \log \frac{q_k}{p_k}.
    \label{eq:rkl_split}
    \end{equation}

    By the definition of the normalized probabilities of non-target classes,
    \begin{equation}
        \hat q_k = \frac{q_k}{\sum_{i \not = m} q_{i}},
        \
        \hat p_k = \frac{p_k}{1-p_m}, \ \text{for all } k \not = m.
        \label{eq:norm_prob_non_target_def}
    \end{equation}
    Substituting it into the second term of Eq.~\eqref{eq:rkl_split} and using $\sum_{k \neq m} \hat q_k = 1$, we obtain
    \begin{align}
        &\sum_{k \neq m} q_k \log \frac{q_k}{p_k} \nonumber \\
        &=
        \sum_{k \neq m} (1-q_m)\hat q_k
        \log \frac{(1-q_m)\hat q_k}{(1-p_m)\hat p_k} \nonumber \\
        &=
        (1-q_m)\sum_{k \neq m} \hat q_k
        \log \left(
        \frac{1-q_m}{1-p_m}\cdot \frac{\hat q_k}{\hat p_k}
        \right) \nonumber \\
        &=
        (1-q_m)\sum_{k \neq m} \hat q_k
        \left(
        \log \frac{1-q_m}{1-p_m}
        +
        \log \frac{\hat q_k}{\hat p_k}
        \right) \nonumber \\
        &=
        (1-q_m)\log \frac{1-q_m}{1-p_m}\sum_{k \neq m} \hat q_k \nonumber \\
        &+
        (1-q_m)\sum_{k \neq m} \hat q_k \log \frac{\hat q_k}{\hat p_k} \nonumber \\
        &=
        (1-q_m)\log \frac{1-q_m}{1-p_m}
        +
        (1-q_m)\sum_{k \neq m} \hat q_k \log \frac{\hat q_k}{\hat p_k} \nonumber \\
        &=
        (1-q_m)\log \frac{1-q_m}{1-p_m}
        +
        (1-q_m)D_{\mathrm{KL}}(\hat q\|\hat p).
        \label{eq:non_target_kl}
    \end{align}

    Substituting Eq.~\eqref{eq:non_target_kl} into Eq.~\eqref{eq:rkl_split} and noticing that $\tilde q_m = (q_m, 1-q_m)$ and $\tilde p_m = (p_m, 1-p_m)$, we have
    \begin{align}
    D_{\mathrm{KL}}(q\|p)
    &=
    q_m \log \frac{q_m}{p_m}
    +
    (1-q_m)\log \frac{1-q_m}{1-p_m} \nonumber \\
    &+
    (1-q_m)D_{\mathrm{KL}}(\hat q\|\hat p) \nonumber \\
    &=
    D_{\mathrm{KL}}(\tilde q_m\|\tilde p_m)
    +
    (1-q_m)D_{\mathrm{KL}}(\hat q\|\hat p). \nonumber
    \label{eq:rkl_before_trkl}
    \end{align}
    This completes the proof.
\end{proof}

\section{Proof of Proposition~\ref{prop:rev_tckd_instability}}
\label{appendix:proof_32}

\begin{proof}
    Since $q= \frac{\exp(z_{s,k})_{k}}{\sum_{i} \exp(z_{s,i})}$, the softmax derivative gives
    \begin{align}
    &\frac{\partial q_m}{\partial z_{s,m}}
    =
    q_m(1-q_m), \\ 
    &\frac{\partial q_k}{\partial z_{s,m}} = -q_k q_m \text{ for all } k \not = m.
    \label{eq:softmax_derivatives}
    \end{align}

    \paragraph{Part I: Gradient of the reverse TRKL term.}
    We have
    \begin{align*}
        \mathcal{L}_{\mathrm{TRKL}}
        &=
        D_{\mathrm{KL}}(\tilde q_m\|\tilde p_m) \\
        &=
        q_m \log \frac{q_m}{p_m}
        +
        (1-q_m)\log \frac{1-q_m}{1-p_m}.
    \end{align*}
    We compute the derivatives of the two terms on the right-hand side as follows
    \begin{align*}
        \frac{\partial}{\partial q_m}
        \left(
        q_m \log \frac{q_m}{p_m}
        \right)
        &=
        \log \frac{q_m}{p_m} + 1,
        \\
        \frac{\partial}{\partial q_m}
        \left(
        (1-q_m)\log \frac{1-q_m}{1-p_m}
        \right)
        &=
        -\log \frac{1-q_m}{1-p_m} - 1.
    \end{align*}

    Combining this with Eq.~\eqref{eq:softmax_derivatives} and the chain rule, we~obtain
    \begin{align}
        \frac{\partial \mathcal{L}_{\mathrm{TRKL}}}{\partial z_{s,m}}
        &=
        \frac{\partial q_m}{\partial z_{s,m}}
        \frac{\partial \mathcal{L}_{\mathrm{TRKL}}}{\partial q_m} \nonumber \\
        &= q_m(1-q_m)
        \log \frac{q_m (1-p_m)}{p_m (1-q_m)}. \nonumber
    \end{align}

    \paragraph{Part II: Gradient of the reverse NRKL term.}
    Using the definition of NRKL and the chain rule, we have
    \begin{align}
        \frac{ \partial \mathcal{L}_{\mathrm{NRKL}}}{\partial z_{s,m}}
        &=
        \frac{\partial \left( (1-q_m) D_{\text{KL}}( \hat{q} \| \hat{p} ) \right)}{\partial z_{s,m}} \nonumber \\
        &= \frac{\partial (1-q_m)}{\partial z_{s, m}} D_{\text{KL}}( \hat{q} \| \hat{p} )  \nonumber \\
        &+ (1-q_m)\frac{\partial D_{\text{KL}}( \hat{q} \| \hat{p} ) }{\partial z_{s,m}}.
        \label{eq:nrkl_chain_rule}
    \end{align}

    Using Eq.~\eqref{eq:softmax_derivatives}, we now show that $\hat{q}$ does not depend on $z_{s,m}$. Indeed, for any $k \not = m$, we have
    \begin{align}
        \frac{\partial \hat{q}_k} {\partial z_{s,m}}
        &= \frac{\partial \frac{q_k}{1-q_m}}{\partial z_{s,m}} \nonumber \\
        &= \frac{(1-q_m) \frac{\partial q_k}{\partial z_{s,m}} - q_k \frac{\partial (1-q_m)}{\partial z_{s,m}}}{(1-q_m)^2} \nonumber \\
        &= \frac{ (1 - q_m) (-q_k q_m) - q_k(-q_m(1-q_m))}{ (1-q_m)^2} \nonumber \\ 
        &= 0. \nonumber
    \end{align}
    By substituting this and Eq.~\eqref{eq:softmax_derivatives} into Eq.~\eqref{eq:nrkl_chain_rule}, we~have
    \begin{equation*}
        \frac{ \partial \mathcal{L}_{\mathrm{NRKL}}}{\partial z_{s,m}} = -q_m(1-q_m) D_{\text{KL}}( \hat{q} \| \hat{p} ).
    \end{equation*}
    This completes the proof.
\end{proof}

\section{Symmetric KL}
\label{appendix:grad_scale}

A natural approach to balance the fidelity of FKL and the diversity of RKL is to interpolate between the two objectives. However, such combinations are often ill-conditioned in practice. The two KL directions induce gradients with fundamentally different scales, so the interpolation coefficient does not reliably control the effective update direction.

\boldhdr{Gradient-scale mismatch}
Consider the weighted objective (Sym-KL) with $\alpha \in (0,1)$:
\begin{equation}
\mathcal{L}_{\mathrm{Sym-KL}}(q)
= \alpha\, D_{\mathrm{KL}}(p \| q)
+ (1-\alpha)\, D_{\mathrm{KL}}(q \| p).
\label{eq:mixture_def}
\end{equation}

\begin{figure}
    \centering
    \includegraphics[width=\linewidth]{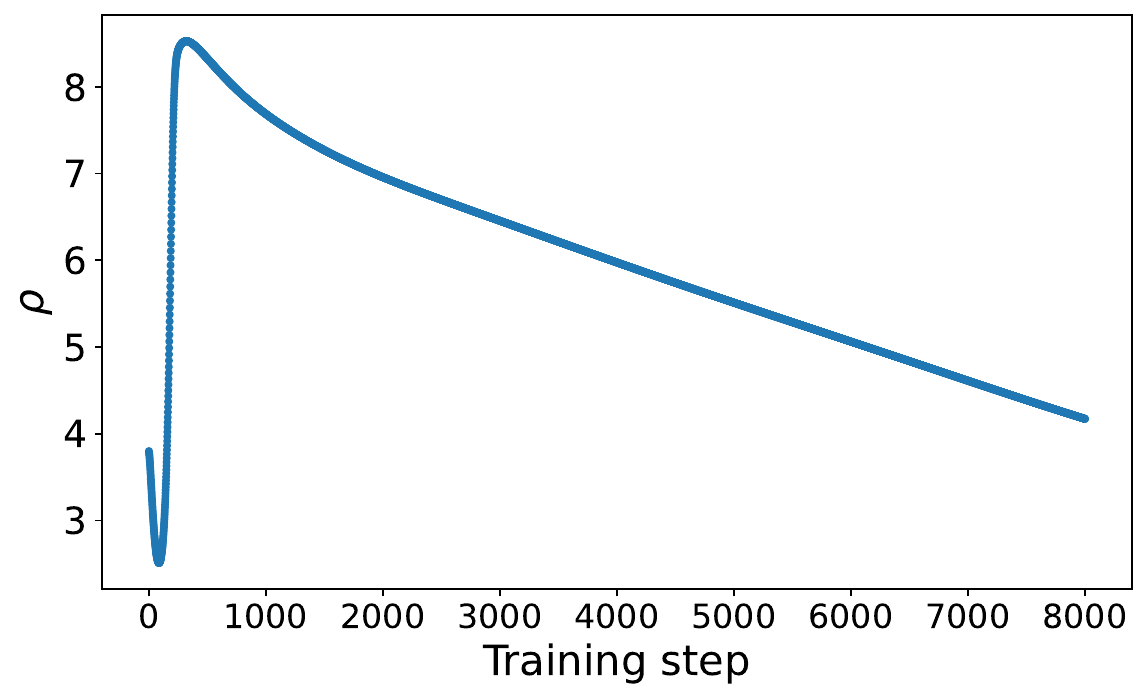}
    \vspace{-0.7cm}
    \caption{Ratio of RKL to FKL gradient norms, illustrating gradient-scale mismatch.}
    \label{fig:grad_ratio}
    \vspace{-0.5cm}
\end{figure}

The gradient of forward KL depends on the probability difference $(q_j - p_j)$ and is bounded in $(-1,1)$. In contrast, reverse KL depends on log-probability ratios $\log(q_j/p_j)$, which can become arbitrarily large when $p_j \rightarrow 0$. This is particularly pronounced in large-vocabulary settings, where many tail tokens have extremely small teacher probabilities.
Consequently, reverse KL gradients on tail tokens can dominate the weighted sum by orders of magnitude, especially in the early stages of training under teacher–student mismatch. This implies that naive loss interpolation is inherently biased toward reverse KL, regardless of the choice of $\alpha$. To quantify this effect, we measure the ratio between reverse and forward KL gradient norms on the toy example:
\begin{equation}
\rho
= \frac{\|\nabla_{\theta} D_{\mathrm{KL}}(q \| p)\|_2}
       {\|\nabla_{\theta} D_{\mathrm{KL}}(p \| q)\|_2}.
\end{equation}
As shown in Figure~\ref{fig:grad_ratio}, $\rho > 1$ in the early stages of training (often exceeding 3 and reaching values as high as 8), confirming a persistent gradient-scale imbalance. Therefore, the weighted sum does not correspond to a meaningful interpolation of optimization behavior unless gradient scales are explicitly controlled.

\begin{figure*}[h]
    \centering
    \begin{subfigure}[t]{0.32\linewidth}
        \centering
        \includegraphics[width=\linewidth]{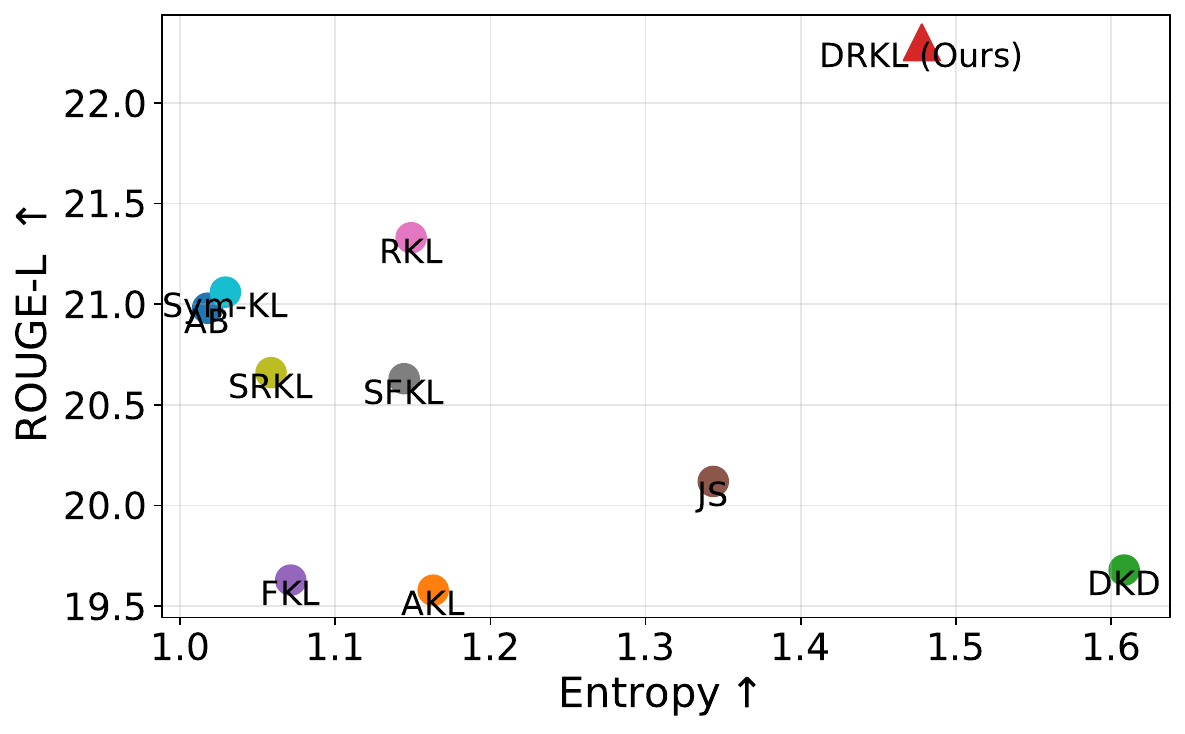}
        \caption{Self-Instruct}
        \label{fig:entropy_vicuna}
    \end{subfigure}
    \hfill
    \begin{subfigure}[t]{0.32\linewidth}
        \centering
        \includegraphics[width=\linewidth]{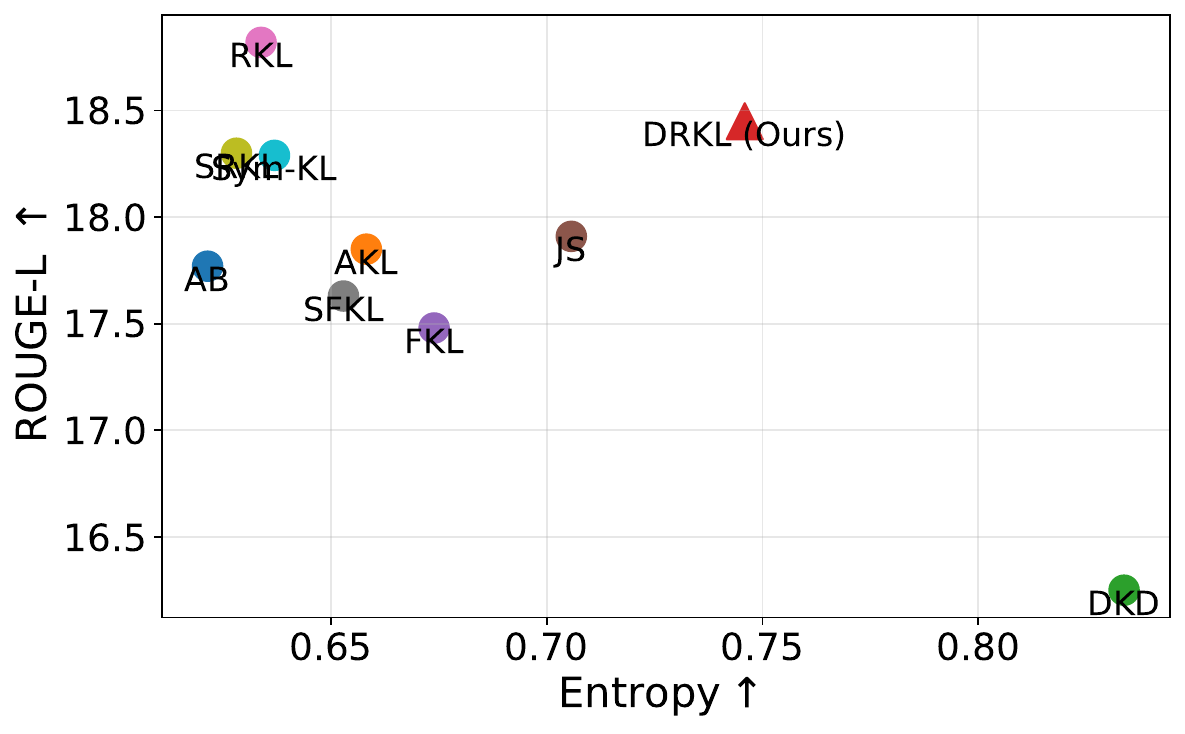}
        \caption{Vicuna}
        \label{fig:entropy_second}
    \end{subfigure}
    \hfill
    \begin{subfigure}[t]{0.32\linewidth}
        \centering
        \includegraphics[width=\linewidth]{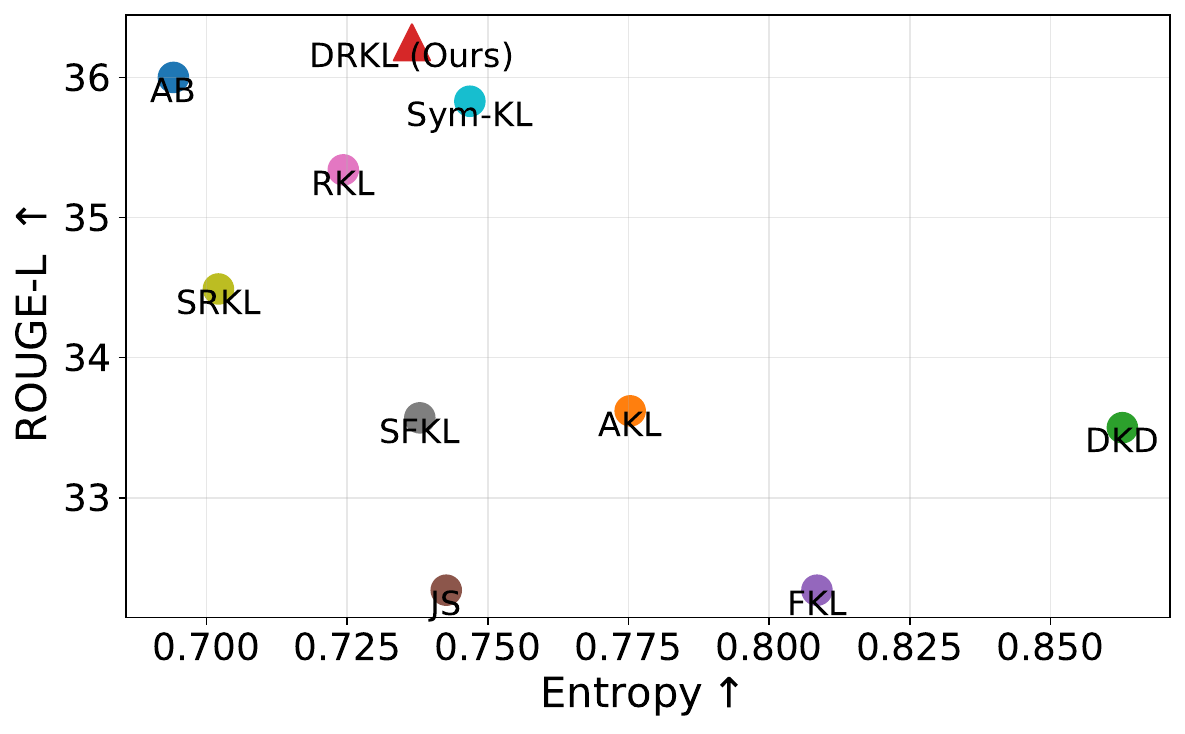}
        \caption{Super-NI}
        \label{fig:entropy_third}
    \end{subfigure}
    \vspace{-2mm}
    \caption{Output entropy versus ROUGE-L on student Qwen3 4B distilled from teacher Qwen3 32B.}
    \label{fig:entropy}
\end{figure*}

\section{DRKL Improves Alignment with Teacher Distribution}\label{appendix:toy_example}

We present controlled toy experiments to provide an intuitive validation of our proposed method in comparison to existing KL-based objectives, focusing on alignment with the teacher distribution. The setup deliberately introduces a capacity mismatch: the teacher distribution is bimodal, while the student is constrained to be unimodal.

\boldhdr{Setup}
The teacher density $p$ is a two-component Gaussian mixture on $\mathbb{R}$.
The student family $q$ is a unimodal Gaussian, which cannot represent both modes simultaneously.
Thus, any optimal solution must trade off mean-seeking (placing mass across both modes) against mode-seeking (placing mass on a single mode).

\begin{figure}
    \vspace{-0.2cm}
    \centering
    \includegraphics[width=\linewidth]{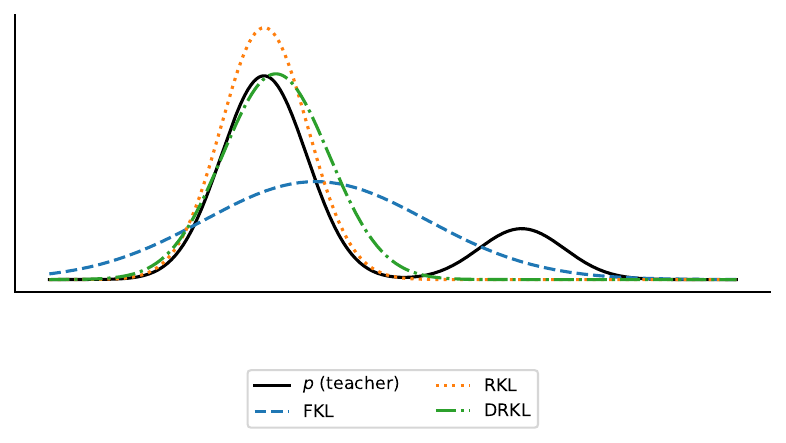}
    \vspace{-0.7cm}
    \caption{Teacher–student alignment under different distillation losses.}    \label{fig:toy_example}
    \vspace{-0.5cm}
\end{figure}

\boldhdr{RKL produces overconfident student}
Figure~\ref{fig:toy_example} shows that students trained with RKL exhibit a sharper, mode-seeking distribution, concentrating probability mass on the dominant mode while suppressing lower-probability modes. As a result, RKL assigns higher confidence to the dominant mode compared to the teacher, as well as to students trained with FKL and DRKL.
This behavior indicates that RKL encourages overconfident predictions by overemphasizing high-probability regions and neglecting the tail distribution. These observations are consistent with our theoretical analysis of RKL's tendency to amplify target confidence through its gradient structure.

\boldhdr{DRKL mitigates overconfidence in RKL}
Figure~\ref{fig:toy_example} shows that students trained with DRKL retain the mode-seeking behavior of RKL while achieving better alignment with the dominant mode of the teacher distribution, without exhibiting the overconfident predictions induced by RKL.
This improvement is consistent with our analysis, as DRKL removes the gradient pathway responsible for confidence amplification by decoupling the target and non-target components. As a result, DRKL preserves the optimization advantages of RKL while avoiding its tendency to amplify target confidence.

\section{Additional Experimental Results}\label{appendix:additionalresults}

\subsection{Entropy Analysis}
\label{app:entropy}

To examine the fidelity--diversity trade-off across distillation objectives, we measure output entropy for the Qwen3 32B $\to$ Qwen3 4B distillation on Self-Instruct, Vicuna, and Super-NI, shown in Figure~\ref{fig:entropy}. DRKL consistently produces higher-entropy student outputs than RKL on all three benchmarks (+28.6\% on Self-Instruct, +17.7\% on Vicuna, +4.9\% on Super-NI). On Self-Instruct and Super-NI, the entropy gain comes together with a ROUGE-L improvement over RKL (+0.97 and +0.91). On Vicuna, RKL has a small ROUGE-L edge (18.82 vs.\ 18.45), but DRKL still gives its largest entropy gain there (+17.7\%) while still outperforming FKL, AKL, and DKD. These results show that DRKL reduces the overconfidence (mode-seeking behavior) that reverse KL alone induces in the student, giving a better fidelity--diversity trade-off than single-term KL-based distillation losses.

\subsection{GPT-2}
\label{app:gpt2}
We additionally evaluate on the GPT-2 XL teacher-student setup, as GPT-2~\citep{radford2019language} is the standard backbone used throughout prior work on LLM distillation~\citep{gu2023minillm, ko2024distillm}, allowing direct comparison with established baselines under matched model scale and training conditions.
The results in Table~\ref{tab:gpt2} show that the proposed DRKL objective consistently outperforms other SOTA distillation methods. DRKL ranks first on all five benchmarks and achieves the highest average performance for both GPT-2 Base and Medium. Baseline methods exhibit inconsistent performance across benchmarks. For example, RKL performs best on Dolly, DKD on Vicuna, SRKL on Super-NI, and AB on UnNI.
In contrast, DRKL delivers consistently strong performance across all tasks, highlighting its robustness. On Super-NI and UnNI, DRKL surpasses RKL by more than $2$ and $2.6$ points, respectively. In Dolly and UnNI, DRKL improves over the second-best more than $0.8$ and $2.0$ points, respectively. A similar trend is observed for GPT-Medium (0.3B parameters), where DRKL ranks first on four out of five benchmarks. It outperforms RKL by $0.37$ to $1.5$ points across all benchmarks and exceeds the second-best method by over $1.5$ points on Dolly.

\begin{table*}[h]
    \centering
    \vspace{-4mm}
   \tiny
    \caption{Comparison of distillation losses for GPT-2 and OPT teacher--student pairs. We report ROUGE-L, with mean and standard deviation computed over five random seeds. The best results are \textbf{boldfaced} and the second-best are \underline{underlined}.}
    \vspace{-3mm}
    \resizebox{\linewidth}{!}{
        \begin{tabular}{lcccccc}
        \toprule
         \textbf{Method}  &  {\scriptsize \textbf{Dolly} }  &  {\scriptsize \textbf{Self-Instruct}} & {\scriptsize \textbf{Vicuna} }  & {\scriptsize \textbf{Super-NI} } & {\scriptsize \textbf{UnNI} } & {\scriptsize \textbf{Avg.} ($\uparrow$)}\\
        \midrule
        \rowcolor{green!10} GPT-2 XL (Teacher) & \rateinline{27.00}{0.19}&\rateinline{14.07}{0.37}&\rateinline{16.31}{0.32}&\rateinline{26.46}{0.41}&\rateinline{31.10}{0.06} &22.99 \\
        \midrule
        \multicolumn{6}{l}{\textbf{\textit{GPT-2 XL (1.5B) $\to$ GPT-2 Base (0.1B)}}} \\
        \midrule

        FKL &  \rateinline{23.80}{0.55} &\rateinline{9.68}{0.49} &\rateinline{14.53}{0.31} &\rateinline{16.27}{0.24} & \rateinline{19.03}{0.09} & 16.66 \\

        RKL{\scriptsize~\citep{gu2023minillm}}
        &  \rateinline{\underline{24.67}}{0.13}
        &\rateinline{11.25}{0.35}
        &\rateinline{15.82}{0.43}
        &\rateinline{21.03}{0.21}
        & \rateinline{23.94}{0.12}
        & 19.34 \\

        Sym-KL
        &  \rateinline{24.38}{0.14}
        &\rateinline{11.07}{0.43}
        &\rateinline{15.72}{0.75}
        &\rateinline{19.92}{0.25}
        & \rateinline{22.21}{0.03}
        & 18.66 \\

        DKD{\scriptsize~\citep{zhao2022decoupled}}
        & \rateinline{22.88}{0.23}
        & \rateinline{10.28}{0.29}
        & \rateinline{\underline{16.12}}{0.31}
        & \rateinline{16.57}{0.25}
        & \rateinline{19.64}{0.13}
        & 17.10 \\

        JS{\scriptsize~\citep{agarwal2024policy}} &\rateinline{24.08}{0.26} &\rateinline{\underline{11.67}}{0.33} &\rateinline{15.47}{0.37} &\rateinline{19.66}{0.23} & \rateinline{22.17}{0.21} & 18.61 \\

        SFKL{\scriptsize~\citep{ko2024distillm}}
        & \rateinline{24.44}{0.42}
        & \rateinline{11.57}{0.15}
        & \rateinline{15.35}{0.30}
        & \rateinline{22.01}{0.23}
        & \rateinline{23.47}{0.12}
        & 19.37 \\

        SRKL{\scriptsize~\citep{ko2024distillm}}
        & \rateinline{24.48}{0.37}
        & \rateinline{10.69}{0.21}
        & \rateinline{14.96}{0.34}
        & \rateinline{\underline{23.25}}{0.36}
        & \rateinline{24.01}{0.05}
        & 19.48 \\

        AKL{\scriptsize~\citep{wu2025rethinking}}
        & \rateinline{21.83}{0.19}
        & \rateinline{9.57}{0.24}
        & \rateinline{13.70}{0.21}
        & \rateinline{15.40}{0.20}
        & \rateinline{18.06}{0.13}
        & 15.71 \\

        AB{\scriptsize~\citep{wang2025abkd}}
        & \rateinline{24.32}{0.29}
        & \rateinline{11.05}{0.21}
        & \rateinline{15.82}{0.58}
        & \rateinline{23.08}{0.19}
        & \rateinline{\underline{24.32}}{0.14}
        & \underline{19.72} \\

        \rowcolor{blue!10}
        DRKL (Ours)
        & \rateinline{\textbf{25.51}}{0.42}
        & \rateinline{\textbf{12.34}}{0.58}
        & \rateinline{\textbf{16.33}}{0.31}
        & \rateinline{\textbf{23.39}}{0.11}
        & \rateinline{\textbf{26.57}}{0.11}
        & \textbf{20.83} \\

        \rowcolor{blue!10}
        DRKL-linear (Ours)
        & \rateinline{25.65}{0.10}
        & \rateinline{11.90}{0.56}
        & \rateinline{16.06}{0.28}
        & \rateinline{22.55}{0.13}
        & \rateinline{26.11}{0.14}
        & 20.45 \\

        \midrule
        \multicolumn{6}{l}{\textbf{\textit{GPT-2 XL (1.5B) $\to$ GPT-2 Medium (0.3B)}}} \\
        \midrule

        FKL
        & \rateinline{\underline{25.10}}{0.33}
        & \rateinline{11.45}{0.30}
        & \rateinline{15.69}{0.28}
        & \rateinline{21.10}{0.32}
        & \rateinline{24.42}{0.21}
        & 19.55 \\

        RKL{\scriptsize~\citep{gu2023minillm}}
        & \rateinline{24.99}{0.26}
        & \rateinline{11.89}{0.24}
        & \rateinline{\underline{16.23}}{0.83}
        & \rateinline{23.58}{0.42}
        & \rateinline{25.84}{0.25}
        & 20.51 \\

        Sym-KL
        & \rateinline{24.77}{0.38}
        & \rateinline{11.48}{0.13}
        & \rateinline{14.32}{0.31}
        & \rateinline{22.55}{0.08}
        & \rateinline{25.36}{0.04}
        & 19.70 \\

        DKD{\scriptsize~\citep{zhao2022decoupled}}
        & \rateinline{24.81}{0.41}
        & \rateinline{10.94}{0.39}
        & \rateinline{15.15}{0.48}
        & \rateinline{18.16}{0.23}
        & \rateinline{22.16}{0.22}
        & 18.24 \\

        JS{\scriptsize~\citep{agarwal2024policy}}
        & \rateinline{24.94}{0.24}
        & \rateinline{\underline{12.75}}{0.29}
        & \rateinline{15.90}{0.24}
        & \rateinline{24.09}{0.37}
        & \rateinline{25.88}{0.21}
        & 20.71 \\

        SFKL{\scriptsize~\citep{ko2024distillm}}
        & \rateinline{\underline{25.10}}{0.22}
        & \rateinline{12.12}{0.32}
        & \rateinline{15.33}{0.37}
        & \rateinline{21.33}{0.16}
        & \rateinline{24.25}{0.12}
        & 19.63 \\

        SRKL{\scriptsize~\citep{ko2024distillm}}
        & \rateinline{23.99}{0.33}
        & \rateinline{12.21}{0.21}
        & \rateinline{14.91}{0.31}
        & \rateinline{\underline{24.28}}{0.20}
        & \rateinline{26.55}{0.17}
        & 20.39 \\

        AKL{\scriptsize~\citep{wu2025rethinking}}
        & \rateinline{24.93}{0.36}
        & \rateinline{11.07}{0.36}
        & \rateinline{15.20}{0.37}
        & \rateinline{19.77}{0.15}
        & \rateinline{21.61}{0.11}
        & 18.52 \\

        AB{\scriptsize~\citep{wang2025abkd}}
        & \rateinline{24.52}{0.06}
        & \rateinline{12.28}{0.05}
        & \rateinline{16.11}{0.46}
        & \rateinline{24.06}{0.21}
        & \rateinline{\textbf{28.37}}{0.12}
        & \underline{21.07} \\

        \rowcolor{blue!10}
        DRKL (Ours)
        & \rateinline{\textbf{26.50}}{0.41}
        & \rateinline{\textbf{12.94}}{0.62}
        & \rateinline{\textbf{16.50}}{0.22}
        & \rateinline{\textbf{24.72}}{0.26}
        & \rateinline{\underline{26.95}}{0.08}
        & \textbf{21.52} \\

    \bottomrule
    \end{tabular}
    }
    \label{tab:gpt2}
\end{table*}

\subsection{Validation Performance}
Figure~\ref{fig:validation_rougeL} shows that DRKL achieves higher validation ROUGE-L early in training, reaching nearly $27$ by the 7th epoch, while competing methods remain below $26$. This advantage persists across later epochs, where DRKL maintains the highest validation performance among all methods. Compared to RKL, DRKL reaches strong performance with fewer training iterations.
This improvement comes from DRKL's decoupled gradient structure, which removes interference between target and non-target components and provides more effective supervision for both.

\begin{figure}[t]
    \centering
    \includegraphics[width=0.9\linewidth]{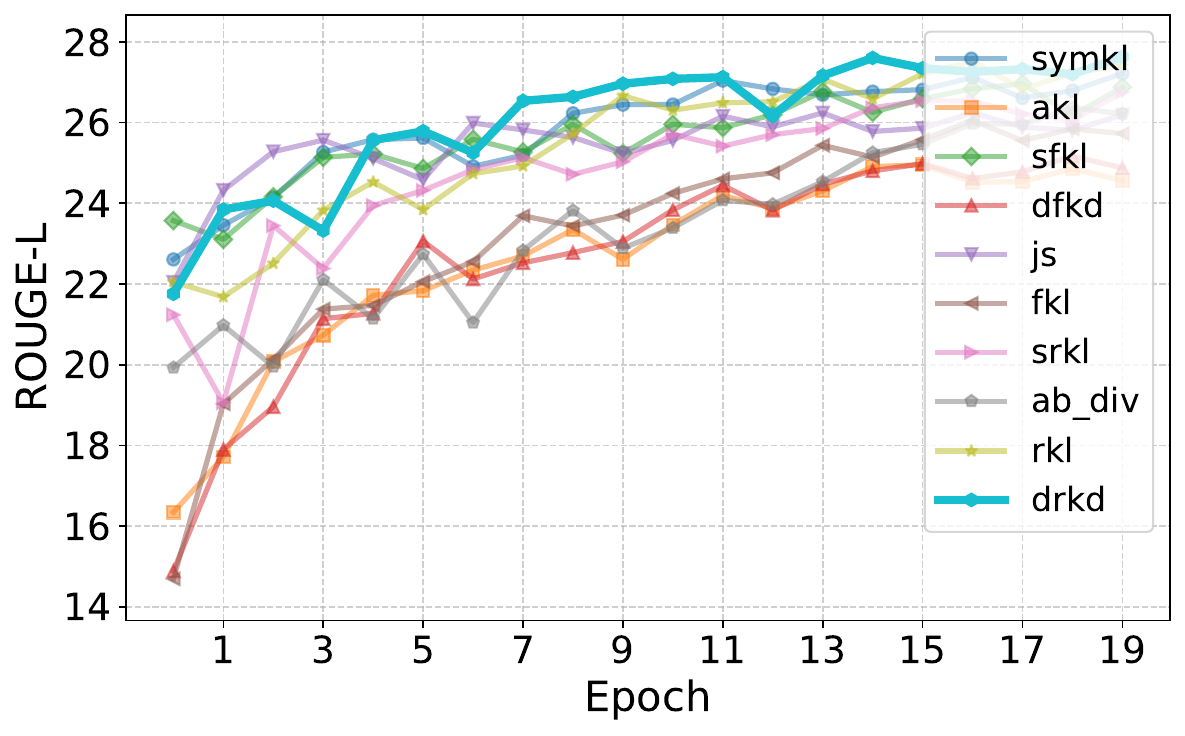}
    \vspace{-4mm}
    \caption{Validation ROUGE-L across training epochs for all distillation objectives. In the legend labels, \texttt{drkd}, \texttt{dfkd}, and \texttt{ab\_div} correspond to DRKL, DKD, and AB, respectively.}
    \label{fig:validation_rougeL}
\end{figure}

\subsection{Adaptive Scheduling of $\gamma$.}

Since $\gamma$ controls the strength of non-target supervision, a natural question is whether scheduling it during training, rather than fixing it, further improves performance. We investigate this with DRKL-linear, a variant that linearly increases $\gamma$ from $0.1$ to $1.0$ over the course of training. Table~\ref{tab:gpt2} reports results on the GPT-2 XL$\to$GPT-2 Base pair. DRKL-linear achieves the best score on Dolly and is the second-best method overall (20.45 average), but its average performance falls slightly short of DRKL with a fixed $\gamma$ (20.83). This suggests that while scheduling $\gamma$ is a promising direction, a simple linear schedule does not by itself outperform a fixed weighting strategy. More broadly, DRKL-linear still outperforms all other baselines, lending further support to our central hypothesis that strengthening non-target supervision improves distillation, while the optimal scheduling strategy likely depends on the schedule shape (e.g., linear, cosine, or stage-wise), the choice of initial and final $\gamma$ values, and the training dynamics of the specific teacher-student pair.

\subsection{Case Studies}
\label{appendix:case_study}

Tables~\ref{table:case_study_unnatural_appendix_1} and~\ref{table:case_study_unnatural_appendix_2} present additional qualitative comparisons of model outputs under different distillation objectives. The student trained with DRKL generates responses that correctly follow diverse instructions and align with the ground truth. For example, it accurately (i) reports exactly what someone said (Case study \#2), (ii) identifies the correct question implied by a paraphrased passage (Case study \#3), (iii) extracts capitalized nouns from the input (Case study \#4), and (iv) fills in blanks based on provided options (Case study \#5). In contrast, students trained with other distillation losses often fail to produce appropriate outputs, frequently misinterpreting the given instructions.

\begin{table*}[t]
\centering
\caption{Qualitative examples from the Unnatural~\citep{honovich2023unnatural} dataset for instruction following. The student OPT-1.3B model is distilled from the OPT-6.7B teacher. DRKL generates responses that better satisfy the given instructions compared to baseline methods.}
\vspace{5pt}
\resizebox{\textwidth}{!}{
\addtolength{\tabcolsep}{2.5pt}
\begin{tabular}{m{3cm}m{20cm}}
\toprule[0.1em]
        \multicolumn{2}{c}{\textbf{\textit{Case study \# 2}}} \\ \midrule[0.1em]
        Instruction & In this task, you will be given a piece of text and asked to report what someone said. Your answer should contain all the necessary information about who said what, but it can be delivered in any format. \\ \midrule
        Input & "I'm so excited for the party tonight!" cried Sarah. \\ \midrule
        \rowcolor{green!10} Ground-truth & Sarah said "I'm so excited for the party tonight!" \\
        \midrule
        \midrule
        FKL    & "I'm so excited for tonight!" cried Sarah \\
        \midrule
        RKL   & Sarah is so excited for the party tomorrow! She lives in the city and is going to have a blast. The weather is finally good, so she can go out and play. She typically brings along her two little crushes, who are her little guards. She plans to buy them each a party band, in memory of her late boyfriend. \\
        \midrule
        Sym-KL    & Text: "Brian! Brian, I'm so excited for the Party tonight!" cried Sarah. \\
        \midrule
        DKD    & "I'm so excited for tonight!" cried Sarah. "If only my boyfriend would stop saying he's going to die \& I wouldn't have to go!" yelled Sarah.\texttt{<n><n>}Sarah moved to the living room, where she saw her boyfriend, Travis, sitting around a poker table.\texttt{<n>}"Hey, bubs, what happened?" Sydney asked.\texttt{<n><n>}Travis responded: "A few guys won a game of poker, and the drank too much," Stephen said.\texttt{<n><n>}Sarah looked angry and began to cry. "I can't handle this!" She said.\texttt{<n><n>}Travis got up from the table and walked around the corner to the living room. "I'm going to have the most epic party ever!" He yelled. Sarah followed him out the door.\texttt{<n><n>}Outside, a crowd was forming. Caleb and Dustin were holding a red puppet, and cheered. Marc, Lauren, and David were playing keep-up on the piano, while Rachels and Mary Jane were dancing nearby. Tom was fighting with his girlfriend, Ann, over a toy, that they were using for a magna carte.\texttt{<n><n>}Harry, Ron, and Justin were outside, while Shai was in the house. An argument broke out about if the party was going \\
        \midrule
        JS  & Sarah was so excited for the party that she couldn't control her excitement. She started texting her friends immediately after she received the text. In her excitement, she wrote an overly dramatic response that made her friends laugh even though she was essentially saying the same thing.<n><n>The truth is, Sarah wasn't very excited about this party; she was literally expecting it to be a productivity dump, where everyone would share their best memes and quotes from episodes of TV shows. Instead, the party was a success, where people laughed and shared their favorite jokes from the shows they loved.  \\ \midrule
        SFKL    & Sarah is so excited for the party that she exhaled loudly. \\
        \midrule
        SRKL    & Sarah: "I'm so thrilled for the party tonight!" exclaimed Emma.\texttt{<n><n>}Emma: "I'm excited for the party, too! My friends and I are going to the party later tonight," replied Sarah.\texttt{<n><n>}Sarah: "I hope it's fun! We're going to theRate The Guys!" exclaimed Emma.\texttt{<n>}There were no references in the text to the time period or place where the event occurred, so your response should be accurate. \\
        \midrule
        AKL    & Sarah is excited for the party tonight! She has been waiting for months to go to a party. The party is tonight! \\
        \midrule
        AB    & Sarah was excited for the party tonight, but didn't know who the person was talking about. \\
        \midrule
        \rowcolor{blue!10} DRKL (Ours)    & Sarah said "I'm so excited for the party tonight!" \\
        \midrule \midrule

        \multicolumn{2}{c}{\textbf{\textit{Case study \# 3}}} \\ \midrule[0.1em]
        Instruction & Given a paragraph with different names for the same thing, write a question that asks about this. \\ \midrule
        Input & Passage: The United States of America is composed of fifty states, each with its own state capital. California's capital is Sacramento, while Florida's state capital is Tallahassee and Texas' capital city is Austin.\\ \midrule
        \rowcolor{green!10} Ground-truth & What is the capital of the United States of America? \\
        \midrule
        \midrule
        FKL    & Yes, the United States of America is made of fifty states, each having its own state capital. The states that comprise the United States of America are California, Florida, Pennsylvania, and Texas. \\
        \midrule
        RKL   & A passage describing the United States of America includes the phrase "each with its own state capital". This questions implies that the United States is composed of anywhere between 50 and 100 states, with each state capital being its own distinct entity. \\
        \midrule
        Sym-KL    & The United States of America is comprised of fifty states, each of which has its own state capital.\texttt{<n>}California's state capital is Sacramento, while Georgia's state capital is Atlanta and Texas' capital city, Austin. \\
        \midrule
        DKD    & The United States of America is comprised of fifty states, each having its own state capital. Sacramento is the capital for California, while Tallahassee and Austin are the capital cities for Florida and Texas respectively. \\
        \midrule
        JS  & What is the distinction between Sacramento and Tallahassee?  \\ \midrule
        SFKL    & What is the name of the state capital for the United States of America? \\
        \midrule
        SRKL    & The United States of America is comprised of fifty states, each featuring its own state capital. Sacramento is California's capital, while Tallahassee is Florida's state capital. \\
        \midrule
        AKL    & What is the name of the United States of America? \\
        \midrule
        AB    & Given a paragraph that has the same words but has different locations, write a question that answers this. \\
        \midrule
        \rowcolor{blue!10} DRKL (Ours)    & What is the capital of the United States of America? \\
        \bottomrule
\bottomrule[0.1em]
\end{tabular}
}\label{table:case_study_unnatural_appendix_1}
\end{table*}

\begin{table*}[ht]
\centering
\caption{Qualitative examples from the Unnatural~\citep{honovich2023unnatural} dataset for instruction following. The student OPT-1.3B model is distilled from the OPT-6.7B teacher. DRKL generates responses that better satisfy the given instructions compared to baseline methods.}
\vspace{5pt}
\resizebox{\textwidth}{!}{
\addtolength{\tabcolsep}{2.5pt}
\begin{tabular}{m{3cm}m{20cm}}
\toprule[0.1em]
        \multicolumn{2}{c}{\textbf{\textit{Case study \# 4}}} \\ \midrule[0.1em]
        Instruction & Given a paragraph, write down all the proper nouns in the order they occur. A proper noun is a specific name for a person, place, or thing and is always capitalized. \\ \midrule
        Input & Passage: Toby thought that Mr. Krebs was being unfair when he gave out detentions for talking during class, but Mrs. Dvorak said that it was part of his job as a teacher.\\ \midrule
        \rowcolor{green!10} Ground-truth & Toby, Mr. Krebs, Mrs. Dvorak \\
        \midrule
        \midrule
        FKL    & Toby, Ms. Dvorak, and Mr. Krebs are all proper nouns. \\
        \midrule
        RKL   & Toby, Mrs. Dvorak, Mr. Krebs, Class \\
        \midrule
        Sym-KL    & Toby, Mrs. Dvorak \\
        \midrule
        DKD    & Toby, Mrs. Dvorak, Mr. Krebs \\
        \midrule
        JS  & Toby, Mrs. Dvorak, and Mr. Krebs Waited for the perfect opportunity to strike back at Mr. Krabs. When he was delivering the class talk, he stopped cold and asked them to exchange their places. Toby assumed that Mr. Krabs was being unfair because he was giving out detentions for speaking during class. Mrs. Dvorak explained that it was part of her job as a teacher to make sure that class was taught properly.  \\ \midrule
        SFKL    & Toby, Mrs. Dvorak \\
        \midrule
        SRKL    & Toby, Mrs. Dvorak, and Mr. Krebs all agreed that Mr. Krebs should have given out detentions for discussing classes during class, but Mrs., Ms., and Mr. Krebs disagreed.  Regarding Mr. Krebs' position on detentions, Mrs. Dvorac stated that it was his job to teach, and therefore he should teach what is required of him.  Regarding Mr. Bacon's position on speech, Mrs. Dvorack said that he was a teacher, and therefore he should impart wisdom and knowledge to his students.  Regarding Mr. Schmidt's position on language, Mrs. Dvoraken stated that he was a professor, and therefore he should use his knowledge to teach what is required of himself. \\
        \midrule
        AKL    & Toby, Mrs. Dvorak, and Mr. Krebs are all proper nouns. \\
        \midrule
        AB    & 1. Toby thought that Mr. Krebs was being unfair when He gave out detentions For talking during class, but Mr. Dvorak said it was part of his Job as a teacher. <n>2. Mrs. Dvorak added that it was his job to teach, and he should know what was expected of him. <n>3. Toby agreed, and added that he would remember this when he got home. \\
        \midrule
        \rowcolor{blue!10} DRKL (Ours)    & Toby, Mr. Krebs, Mrs. Dvorak \\
        \midrule \midrule

        \multicolumn{2}{c}{\textbf{\textit{Case study \# 5}}} \\ \midrule[0.1em]
        Instruction & You will be given a piece of text that is missing important information, and a list of words. Your job is to fill in the blank with one of the words from the list so that the sentence/paragraph makes grammatical sense. \\ \midrule
        Input & I asked if he wanted to go for coffee, but he said \_\_\_\_\_.\texttt{<n>}Word options: ['no', 'not', 'didnt'] \\ \midrule
        \rowcolor{green!10} Ground-truth & I asked if he wanted to go for coffee, but he said no. \\
        \midrule
        \midrule
        FKL    & "I asked if he wanted to going for coffee, but he says \_\_\_\_."\texttt{<n><n>}Thus, "I asked if he wanted go for coffee, but [name] said \_\_\_\_." This constitutes fixing the grammatical error in the sentence "I asked if he wanting to go for coffee, \_\_\_\_. \\
        \midrule
        RKL   & I asked if he wanted to going for coffee, but he says \_\_\_\_\_. \\
        \midrule
        Sym-KL    & word options: 'no', 'not','didnt' \\
        \midrule
        DKD    & Philip did not want to go for coffee. He said that he did not want to go to coffee because it was unhealthy. \\
        \midrule
        JS  & He would rather go for a walk.  \\ \midrule
        SFKL    & The sentence is grammatically incorrect because it should read as:\texttt{<n>}I asked if he wanted to take a coffee break, but he said \_\_\_\_.\texttt{<n>}The correct sentence is:\texttt{<n>}I asked if He wanted to take a coffee walk, but he said \_\_\_\_\_\_\_.\texttt{<n>}The grammatical problem with the statement is that it is missing the crucial phrase - hence it is incorrect grammatically. \\
        \midrule
        SRKL    & Because you provided no information in the sentence, and only provided a list of words, the context is deeply misunderstood, and your answer is incorrect. You should have provided more information in the text to clarify the context. \\
        \midrule
        AKL    & Iamar did not want to go for coffee; he said [."not", "didnt", "not"]. \\
        \midrule
        AB    & These statements are correct, but I would also recommend using the word ``to'' consistently in this case.
        In the first statement, ``The glass fell three stories'' is a complete statement, while ``It broke into pieces'' is also a complete statement.
        In the second statement, ``A piece hit John on the head'' is a correct statement, whereas ``A piece hit John onthe head'' and ``A piece hit Johnon thehead'' contain spacing errors and should be corrected. \\
        \midrule
        \rowcolor{blue!10} DRKL (Ours)  & The correct answer is 'no'. Unless he specifically asks for a coffee, he doesn't want to go for coffee. \\
        \bottomrule
\bottomrule[0.1em]
\end{tabular}
}\label{table:case_study_unnatural_appendix_2}
\end{table*}

\section{LLM-as-a-judge Judge Prompt}
\label{app:laaj_prompt}

Figure~\ref{fig:laaj_prompt} shows the prompt used to obtain LLM-as-a-judge scores in Table~\ref{tab:laaj}, adapted from the MT-Bench single-answer grading protocol of \citet{zheng2023judging}.

\begin{figure}[t]
\centering
\begin{minipage}{0.95\linewidth}
\small
\ttfamily
\textbf{System prompt:}

Please act as an impartial judge and evaluate the quality of the response provided by an AI assistant to the user question displayed below. Your evaluation should consider factors such as the helpfulness, relevance, accuracy, depth, creativity, and level of detail of the response. Begin your evaluation by providing a short explanation. Be as objective as possible. After providing your explanation, you must rate the response on a scale of 1 to 10 by strictly following this format: "[[rating]]", for example: "Rating: [[5]]".

\vspace{6pt}
\textbf{User message:}

{[}Question{]}\\
\{instruction\}

\{input\}

\vspace{4pt}
{[}The Start of Assistant's Answer{]}\\
\{response\}\\
{[}The End of Assistant's Answer{]}
\end{minipage}
\caption{LLM-as-a-judge prompt, adapted from the MT-Bench single-answer grading protocol~\citep{zheng2023judging}. The judge scores the model's response and the ground-truth reference independently (each with a separate call to this same prompt, substituting the response under evaluation); we report the ratio of the two scores, scaled to 0--100.}
\label{fig:laaj_prompt}
\end{figure}

\section{Use of AI Assistants}
\label{appx:llms}

In preparing this paper, we made limited use of AI Assistants such as ChatGPT only to help refine wording, correct grammar, and enhance clarity. All core research contributions, including conceptual development, theoretical analysis, methodology design and implementation, experimental evaluations, and paper organization and composition, were conducted entirely by the authors.

\end{document}